\documentclass[letterpaper,11pt]{article}
\usepackage[margin=1in]{geometry}
\usepackage[utf8]{inputenc} 
\usepackage[T1]{fontenc}    
\usepackage{xcolor}         
\usepackage{amsmath,amsthm,amsfonts,amssymb}
\usepackage{thm-restate}
\usepackage{soul}
\usepackage{ifthen}
\usepackage{cite}
\usepackage{url,color,cite}
\usepackage{cases}
\usepackage{mathtools}
\DeclarePairedDelimiter{\ceil}{\lceil}{\rceil}
\DeclarePairedDelimiter{\floor}{\lfloor}{\rfloor}
\usepackage{float}
\usepackage{nicefrac}
\usepackage{algorithm}
\usepackage[noend]{algpseudocode}
\usepackage{pifont}
\usepackage{subcaption}
\usepackage{sidecap}
\usepackage{multirow}
\usepackage{multicol}
\usepackage{bbm}

\usepackage{soul}

\newtheorem{assumption}{Assumption}
\newtheorem{lemma}{Lemma}
\newtheorem*{lemma*}{Lemma}
\newtheorem{theorem}{Theorem}
\newtheorem*{theorem*}{Theorem}
\newtheorem{proposition}{Proposition}

\newtheorem{corollary}{Corollary}
\theoremstyle{definition}

\newcommand{\adot}[1]{{\langle #1 \rangle}}
\newcommand{\flr}[1]{{\lfloor #1 \rfloor}}
\newcommand{\cblue}[1]{{\color{blue} #1}}

\newcommand{\ind}{\mathbf{1}}
\newcommand{\bbE}{\mathbb{E}}
\newcommand{\bbN}{\mathbb{N}}
\newcommand{\bbR}{\mathbb{R}}
\newcommand{\calA}{\mathcal{A}}
\newcommand{\calD}{\mathcal{D}}
\newcommand{\calF}{\mathcal{F}}
\newcommand{\calM}{\mathcal{M}}
\newcommand{\calH}{\mathcal{H}}
\newcommand{\calS}{\mathcal{S}}
\newcommand{\calX}{\mathcal{X}}
\newcommand{\thetas}{\theta_\star}

\newcommand{\eps}{\epsilon}

\usepackage{microtype}
\usepackage{graphicx}
\usepackage{epstopdf}
\usepackage{booktabs} 
\usepackage{hyperref}
\usepackage{cleveref}
\usepackage{enumitem}

\usepackage{lipsum}

\newcommand\blfootnote[1]{%
	\begingroup
	\renewcommand\thefootnote{}\footnote{#1}%
	\addtocounter{footnote}{-1}%
	\endgroup
}

\begin{document}

\title{Contexts can be Cheap: Solving Stochastic Contextual Bandits with Linear Bandit Algorithms\blfootnote{This work is accepted in COLT (Conference on Learning Theory), 2023.}}

 \author{Osama A. Hanna$^\dagger$, Lin F. Yang$^\dagger$ and Christina Fragouli$^\dagger$\\ 
 $^\dagger$University of California, Los Angeles\\
 Email:\{ohanna, linyang, christina.fragouli\}@ucla.edu 
 }
\date{}
\maketitle

\begin{abstract}
    In this paper, we address the stochastic contextual linear bandit problem, where a decision maker is provided a context (a random set of actions drawn from a distribution). The expected reward of each action is specified by the inner product of the action  and an unknown parameter. The goal is to design an algorithm that learns to play as close as possible  to the unknown optimal policy after a number of action plays. This problem is considered more challenging than the linear bandit problem, which can be viewed as a contextual bandit problem with a \emph{fixed} context. Surprisingly, in this paper, we show that the stochastic contextual problem can be solved as if it is a linear bandit problem. In particular, we establish a novel reduction framework that converts every stochastic contextual linear bandit instance to a linear bandit instance, when the context distribution is known. When the context distribution is unknown, we establish an algorithm that reduces the stochastic contextual instance to a sequence of linear bandit instances with small misspecifications and achieves nearly the same worst-case regret bound as the algorithm that solves the misspecified linear bandit instances. 
     As a consequence, our results imply a $O(d\sqrt{T\log T})$ high-probability regret bound for contextual linear bandits, making progress in resolving an open problem in \cite{li2019nearly, li2021tight}.
     Our reduction framework opens up a new way to approach stochastic contextual linear bandit problems, and enables  improved regret bounds in a number of instances including the batch setting, contextual bandits with misspecifications, contextual bandits with sparse unknown parameters, and contextual bandits with adversarial corruption.

\end{abstract}
\allowdisplaybreaks

\section{Introduction}
\begin{table*}[t!]
\centering
\begin{tabular}{|c|c|c|c|}
\hline
  Algorithm & Regret Bound & Type & Assumption/restriction \\ 
 \hline\hline
\cite{abbasi2011improved}& $O(d\sqrt{T}\log T)$ & w.h.p. &  \\
\cite{li2021tight}& $O(d\sqrt{T\log T}poly(\log\log T))$ & exp & \\ 
\cite{li2021tight}& $O(d\sqrt{T}\log T poly(\log\log T))$ & w.h.p. & \\ 
 \textbf{Ours} & $O(d\sqrt{T\log T})$ & w.h.p. & \\
 \hline\hline
 \cite{ruan2021linear}& $O(d\sqrt{T\log (d) \log (T)}\log\log T)$ & exp & batch learning with\\
\textbf{Ours}& $O(d\sqrt{T\log (T)}\log\log T)$ & w.h.p. & $O(\log\log T)$ batches\\
\hline\hline
 \cite{foster2020adapting}& $O(d\sqrt{T}\log T+\epsilon \sqrt{d}T)$ & exp & misspecified\\
\textbf{Ours}& $O(d\sqrt{T\log T}+\epsilon\sqrt{d}T\log T)$ & w.h.p. &\\
\hline\hline
 \cite{foster2020adapting}& $\tilde{O}(d^{4.5}\sqrt{T}+d^4C)$ & w.h.p. & adversarial corruption\\
\textbf{Ours}& $\tilde{O}(d\sqrt{T}+d^{3/2}C)$ & w.h.p. &\\
\hline\hline
 \cite{abbasi2012online}& $O(\sqrt{dsT}\log T)$ & w.h.p. & sparse\\
\textbf{Ours}& $O(\sqrt{dsT\log T})$ & w.h.p. &\\
\hline\hline
\textbf{Ours}&$O(\sqrt{dsT\log T}\log \log T)$&w.h.p.&sparse with $O(\log\log T)$ batches\\
\hline
\end{tabular}
\caption{Comparison of best known in literature vs. our approach regret bounds. Here, {$d$ is the model dimension, $T$ is the time horizon, $\epsilon$ is an upper bound on the amount of misspecification, $C$ limits the power of adversary,  $s$ is an upper bound on the number of non-zero elements in the unknown parameter, exp indicates a regret bound in expectation,  w.h.p. indicates a regret bound that holds with probability at least $1-1/T$, and $\tilde{O}$ hides $\log$ factors.
}}\label{table:res}
 \end{table*}
Linear bandit and contextual linear bandit problems are attracting extensive attention  - for example, more than 
$17,000$ papers appear when searching for  ``linear contextual bandit''  on Google Scholar during the last 5 years
- as they enable to support impactful active learning applications through elegant formulations. 
In linear bandits, a learner at  each time $t\in [T]$, where $T$ is the time horizon,  pulls an arm $a_t$ from a fixed action space $\calA$ (that may be continuous or discrete),   and receives a reward $r_t=\adot {a_t,\thetas} + \eta_t$, where $\thetas$ is an unknown $d$-dimensional vector of parameters and $\eta_t$ is random noise. 
Contextual linear bandits add another layer of complexity by enabling at each round the action space to be different to capture context; in this case, the learner at time $t$ observes an action space (context) $\calA_t$. 
That is, we can think of linear bandits as single-context contextual bandits, observing $\calA_t=\calA$  for all $t$.  
For example, while linear bandits are used in recommendation systems where the set of actions is fixed and
oblivious to the individual the recommendation is addressed to, contextual linear bandits are used in personalized recommendations, where the action space gets tailored to context attributes such as age, gender, income and interests of each individual.  
 
It is not surprising that, although more limited in applications, linear bandits are much better understood in theory than contextual linear bandits.
Indeed, algorithms for linear bandits often leverage the fixed action space property, 
and cannot be easily extended to contextual linear bandits. 
To give a concrete example, the algorithm Phased Elimination (PE) \cite{lattimore2020learning, valko2014spectral} leverages the fixed action space by exploring a (small) core set of actions   to achieve good estimates of the rewards for all actions. This algorithm achieves a high probability regret  bound of $O(d\sqrt{T\log T})$. 
Nevertheless, despite several attempts over the last decade \cite{abbasi2011improved, li2021tight, li2019nearly}, the best known regret upper bounds for contextual linear bandits have a $\log$ (or iterated $\log$) multiplicative gap over the $O(d\sqrt{T\log T})$ bound both in high probability and in expectation. Similarly,
the best known algorithms for several linear bandits problems (e.g., with misspecification, adversarial corruption, and others \cite{lattimore2020learning, foster2020adapting, bogunovic2021stochastic, wei2022model, ruan2021linear}), perform better (in the worst-case) than the corresponding algorithms for contextual linear bandits.
\subsection{Our Results}

\subsubsection*{The Reduction}
We show in this paper the surprising result that, provided the context comes from a distribution $\calD$ (stochastic context), 
contextual linear bandit problems can be reduced to solving  (single context) linear bandit problem when the context distribution $\calD$ is known, 
and to linear bandits with $\tilde{O}(1/\sqrt{T})$ misspecification when the distribution $\calD$  is unknown. These results are presented in the following informal theorems and their exact statements are given in Theorems~\ref{thm:reduction}, \ref{thm:main} and \ref{thm:product}.

\begin{theorem*}[Informal Statement of Theorem~\ref{thm:reduction}]
For any contextual linear bandit instance $I$ with known context distribution $\calD$, there exists (constructively) a linear bandit instance L with the same action dimension, and any algorithm solving L
solves $I$ with the same worst-case regret bound as L.
\end{theorem*}
\vspace*{-0.1cm}
\begin{theorem*}[Informal Statement of Theorem~\ref{thm:main}]
For any contextual linear bandit instance $I$ with unknown context distribution $\calD$, there exist (constructively) $\log T$ misspecified linear bandit instances $L_1,...,L_{\log T}$, where $L_i$ operates on part of the horizon of length $T_i$, has $\tilde{O}(1/\sqrt{T_i})$-misspecification and the same action dimension, and any algorithm solving $L_1,...,L_{\log T}$ 
solves $I$ with the same worst-case regret bound as $L_1,...,L_{\log T}$.
\end{theorem*}
\begin{theorem*}[Informal Statement of Theorem~\ref{thm:product}]
    For any contextual linear bandit instance $I$ with unknown context distribution $\calD$ but where now the action space has a component-wise product structure, there exists (constructively) a linear bandit instance L with double the action dimension of $I$, and any algorithm solving L
solves $I$ with the same worst-case regret bound as L.
\end{theorem*}
Stochastic contextual bandits encompass practical cases where the context is not selected adversarially;
in our example of personalized recommendations, the age, gender, income, come from distributions.
Our framework simplifies the contextual linear bandit problem and enables to use any existing (or future) linear bandit algorithms to solve the contextual case. Moreover, our results separate the case of stochastic contexts from the harder case of adversarial contexts and explain why good results, which are not achievable for adversarial contexts, are possible for stochastic contexts.

\subsubsection*{Implications and Related Work}
The equivalence we proved opens up a new way to approach stochastic contextual linear bandit problems, and results in a number of new results or recovery of existing results in a simpler manner; we next present some of these implications (summarized in Table~\ref{table:res}) and their positioning with respect to related work. In the discussion next,  $d$ is the model dimension and $T$ is the time horizon.\\

\noindent{\bf $\bullet$ Tighter Regret Bounds.} To the best of our knowledge, there is a gap in the regret bounds of contextual linear bandits: the state-of-the-art lower bounds are, $\Omega(d\sqrt{T})$ for linear bandits \cite{lattimore2020bandit} and $\Omega(d\sqrt{T\log T})$ for linear contextual bandit with adversarial contexts \cite{li2019nearly}.\\ 
\underline{\bf  Our contribution  [Corollary~\ref{cor:implication-tighter}].} Our approach achieves a regret upper bound  $O(d\sqrt{T\log T})$ with high probability even when the action set is infinite. While it is not known  if a $\Omega(d\sqrt{T\log T})$ lower bound holds for stochastic contexts, our result improves over state of the art high probability bounds by at least a factor of $\sqrt{\log T}$ and matches the best known upper bound for linear bandits. \\
\underline{\bf  Related Work.}
The best attempts of upper bounds are \cite{abbasi2011improved, li2021tight, li2019nearly}.
In particular, \cite{abbasi2011improved} achieves a regret bound of $O(d\sqrt{T}\log T)$ with high probability;  \cite{li2021tight} achieves  $O(d\sqrt{T\log T}poly(\log\log T))$ in expectation and $O(d\sqrt{T}\log T poly(\log\log T))$ with high probability; and  \cite{li2019nearly} achieves a regret bound of \\ $O(d\sqrt{T\log T}poly(\log\log(T)))$ \emph{in expectation} and only when the number of actions is \emph{finite and bounded by $2^{d/2}$}.\\

\noindent{\bf $\bullet$ Batch Learning.}
In batch learning, instead of observing the reward at the end of each round to decide what action to play next, the learning agent is constrained to split the rounds  into a fixed number of $M$  batches, use a predetermined policy within each batch, and it can only observe the action outcomes and switch the policy at the end of each batch.  
This is a central problem in online learning \cite{gao2019batched,perchet2016batched,ruan2021linear, han2020sequential} as limited policy adaptivity enables parallelism and facilitates deployment of learning algorithms for large-scale models. \\
\underline{\bf  Our Contribution [Theorem~\ref{lem:batch}].} Our batch algorithm for contextual linear bandits achieves with \emph{high probability}  $O(d\sqrt{T\log (T)}\log\log T)$ regret bound for $O(\log \log T)$ batches, a $\sqrt{\log d}$ better than the \emph{in expectation} regret bound in \cite{ruan2021linear}.\\
\underline{\bf Related Work.} A number of works have explored batched contextual linear bandits,
both in adversarial \cite{abbasi2011improved, han2020sequential} and non-adversarial settings \cite{ruan2021linear, zhang2021almost}. 
The breakthrough work in \cite{ruan2021linear} achieved a nearly optimal \emph{in expectation} regret upper bound  $O(d\sqrt{T\log d \log (T)}\log\log T)$  using $O(\log \log T)$ batches for the stochastic contexts setting. {The near-optimality follows 
from the result in~\cite{gao2019batched} which shows that $\Omega(\log \log T)$ batches are required to achieve a $O(\sqrt{T})$ worst-case regret bound for multi-armed bandits with a finite number of arms (this is a special case of contextual linear bandits). If the contexts are chosen adversarially, $\Omega(\sqrt{T})$ batches are required to achieve a $O(\sqrt{T})$  regret bound \cite{han2020sequential}.\\

\noindent{\bf $\bullet$ Misspecified Bandits.} 
Linear bandit algorithms are designed under the assumption that the expected rewards are perfectly linear functions of the actions; misspecified bandits relax this assumption by considering perturbations of the linear model measured by the amount of deviation in the expected rewards (we call the case $\epsilon$-misspecification if the deviation is upper bounded by $\epsilon$).
The non-linearity in the model enables to better capture real-world environments and is of high interest in the literature \cite{du2019good,crammer2013multiclass,lattimore2020learning, foster2020adapting, ghosh2017misspecified, foster2020beyond}.\\
\underline{\bf  Our Contribution [Theorem \ref{thm:mis}, \ref{thm:mis-unknown}].} We provide a regret bound of $O(d\sqrt{T\log T}+\epsilon\sqrt{d}T\log T)$ with high probability for contextual  bandits with unknown misspecification, and $O(d\sqrt{T\log T}+\epsilon\sqrt{d}T)$ with high probability for $\epsilon$ known.
To the best of our knowledge, these results offer the first \emph{optimal regret bounds} for 
$\epsilon$-misspecified contextual linear bandits, and improve over existing literature for unknown $\epsilon$ by providing high probability bounds and improved $\log$ factors. We also present the first nearly optimal algorithm for misspecified contextual linear bandits with $O(\log T)$ batches.\\
 \underline{\bf Related Work.} The work in \cite{lattimore2020learning} shows that PE with modified confidence intervals achieves the optimal regret bound of $O(d\sqrt{T\log T}+\epsilon \sqrt{d}T)$ with high probability (matching the $\Omega(\epsilon\sqrt{d}T)$ lower bound \cite{lattimore2020learning}) for linear bandits with known $\epsilon$ misspecification. If $\epsilon$ is unknown, the same algorithm was shown to achieve a regret bound of $O(d\sqrt{T\log T}+\epsilon \sqrt{d}T\log T)$ with high probability. The work in \cite{zanette2020learning, lattimore2020learning} proposed variants of LinUCB that achieve a regret bound of $\tilde{O}(d\sqrt{T}+\epsilon\sqrt{d}T)$\footnote{ Here $\tilde{O}$ hides $\log$ factors.} with high probability for contextual linear bandits with known $\epsilon$. However,  changing action sets with unknown $\epsilon$ was left as an open problem. The works in \cite{pacchiano2020model,foster2020adapting} made progress in answering this question by providing a regret bound of $O(d\sqrt{T}\log T+\epsilon \sqrt{d}T)$ in expectation \cite{foster2020adapting}. However, improving the $\log$ factors in the first term and strengthening the result to a high probability bound was unresolved. Due to the challenge that changing action sets impose on the analysis, the techniques of \cite{foster2020adapting} cannot be directly extended to provide high probability bounds on the regret. Our result of $O(d\sqrt{T\log T}+\epsilon\sqrt{d}T\log T)$ regret  removes the $\sqrt{\log T}$ factor from the first term and adds a $\log T$ factor in the second term; depending on the value of $\epsilon$ this can lead to a tighter or looser bound.\\


\noindent{\bf $\bullet$ Bandits with Adversarial Corruption.} Linear bandits with adversarial corruption recently attracted significant interest \cite{lykouris2021corruption, chen2021improved, jin2021best, li2019stochastic, bogunovic2020corruption, bogunovic2021stochastic, lee2021achieving, wei2022model} due to the vulnerability of online learning applications to attacks. There are multiple adversarial models that are proposed in the literature; here we consider a widely used model that assumes the adversary  knows the policy, and observes the history, but does not observe the current action before corrupting the rewards. \\
\underline{\bf  Our Contribution [Theorem~\ref{thm:ad}].} Our work provides the first algorithm for contextual linear bandits that achieves a regret bound of $\tilde{O}(d\sqrt{T}+d^{3/2}C)$ with high probability for unknown $C$, which upper bounds the total amount of corruption from the adversary. This 
improves over the best known $\tilde{O}(d^{4.5}\sqrt{T}+d^4C)$ bound for linear bandits with changing action sets in \cite{wei2022model}. We note that in our regret bound, while the dependency on $d$ in the first term is nearly optimal, the dependency on $d^{3/2}$  in the second term is not. However, we simplify the problem of improving this dependency, as any algorithm that improves it for linear bandits will imply the same improvement for contextual  bandits with our reduction.\\
\underline{\bf Related Work.} 
The work in \cite{bogunovic2021stochastic} considers linear bandits with adversarial corruption and achieves a regret bound of $\tilde{O}(d\sqrt{T}+d^{3/2}C)$  with high probability for known corruption level $C$ and a regret bound of $\tilde{O}(d\sqrt{T}+d^{3/2}C+C^2)$ with high probability for a stronger adversary that observes the current action and unknown corruption level $C$, while the work in \cite{wei2022model} achieves a regret bound of $\tilde{O}(d\sqrt{T}+d^{3/2}C)$ with high probability for unknown corruption $C$. These algorithms have optimal dependency on $T,C$ but it is not known if the $d^{3/2}$ dependency is tight or not. The work in \cite{lee2021achieving} improves the dependency on $d$ for linearized corruption, by achieving a regret bound of $\tilde{O}(d\sqrt{T}+C)$ with high probability which is also nearly optimal \cite{lee2021achieving}. The proposed algorithms and analysis rely on the assumption that the action set is fixed, and as far as we know there are no known generalizations for changing action sets - beyond the work in \cite{bogunovic2021stochastic} that considers changing action sets but imposes a strong assumption on the context distribution. The work in \cite{he2022nearly} considers a stronger adversary, that observes the current action, and achieves an $\tilde{O}(d\sqrt{T})$ regret bound for unknown $C\leq \sqrt{T}$ and linear regret otherwise. The first paper to prove a regret bound with nearly optimal dependency on $T,C$ for linear bandits with changing action sets (linear contextual bandits) is \cite{wei2022model} which achieves a regret bound of $\tilde{O}(d^{4.5}\sqrt{T}+d^4C)$ with high probability for unknown $C$. While the dependency on $T,C$  is nearly optimal, the dependency on $d$ is clearly not - improving this was left as an open problem. Our results take a step in this direction by removing a factor of $d^{3.5}$ from the first term, by removing a factor of $d^{2.5}$ from the second term and by reducing the problem of further improving the dependency of $d$ in the second term to achieving this improvement over a linear bandit setup.\\

\noindent{\bf $\bullet$ Sparsity.} 
 High dimensional linear bandits with sparsity capture practical cases such as, when there exist a large number of candidate features and limited information on which of them are useful; use cases include personalized medicine and online advertising \cite{bastani2020online}.
 This setup results in sparsity in the unknown linear bandit parameters, which can be leveraged for more efficient learning. \\
\underline{\bf  Our Contribution [Theorem~\ref{thm:sp}].} Our work provides the first $O(\sqrt{dsT\log T})$ regret bound with high probability for contextual linear bandits, improving a factor of $\sqrt{\log T}$ on the state of the art, where $s$ is an upper bound of the nonzero elements in the model.\\
\underline{\bf Related Work.} 
Due to its practical significance, a number of works have examined this setup  \cite{lattimore2015linear, abbasi2012online, gerchinovitz2011sparsity, carpentier2012bandit, chen2022nearly, dai2022variance, jang2022popart, hao2021information, hao2020high}. 
To the best of our knowledge, the best known regret bound is $O(\sqrt{dsT}\log T)$ with high probability \cite{abbasi2012online}. 
While this is shown to be nearly optimal \cite{lattimore2020bandit}, improving the $\log T$ factor was left as an open problem. Our work resolves this. 

\subsection{Technical Overview} 
Our major technique innovation is in the conception and execution of a reduction from a stochastic contextual linear bandit instance to a linear bandit instance. 
This reduction is made possible by establishing a linear bandit action for each possible parameter $\theta$ of the contextual bandit instance.
In particular, for a given $\theta$, we establish a new action $g(\theta)$ that is the expected best action (under the distribution of the context $\calA_t$) with respect to the parameter $\theta$, i.e., $g(\theta) = \mathbb{E}_{\calD}(\arg\max_{\tilde{a}\in\calA_t}\langle \tilde{a}, \theta\rangle)$.
Note that $g(\theta)$ may not correspond to any valid action for the present context in the corresponding contextual bandit instance. 
Yet, we show that if one plays an action $\arg\max_{\tilde{a}\in\calA_t}\langle \tilde{a}, \theta\rangle$ that is optimal with respect to $\theta$, then the contextual bandit instance generates a \emph{linear reward} with respect to $g(\theta)$. Moreover, the linear bandit instances share the same optimal parameter $\thetas$ as in the contextual instance.
Under standard boundedness assumptions of the contexts and actions, the reward noise in the reduced linear instance and the contextual bandit instance also share a similar sub-Gaussian tail.
By mapping the linear bandit action $g(\theta)$ to the contextual bandit $\arg\max_{\tilde{a}\in\calA_t}\langle \tilde{a}, \theta\rangle$ for any context $\calA_t$ and $\theta$,  any algorithm for the linear bandit problem can be immediately applied to solve for the contextual bandit problem and suffer no additional regret in the worst-case.

The reduction becomes more challenging when the context distribution is unknown. One idea is to estimate $g(\theta)$ for all possible $\theta$. Unfortunately, doing so would require a large number of samples resulting in unbounded regret. 
We resolve this issue by a batched approach where the batches provide increasingly better estimates of   $g(\theta)$. 
In each batch, we estimate $g(\theta)$ using all the contexts generated in the previous batch. 
Note that this inevitably introduces error which ruins the linearity of the collected reward from the contextual bandit instance.
Hence, we can only apply algorithms that are designed for misspecified linear bandits. 
Luckily, with a carefully designed batch sequence, we show that a linear bandit algorithm that works for all misspecification levels $\epsilon\in[1/\sqrt{T}, 1]$ can be applied to solve the contextual bandit instance. As it is hard to guarantee a good estimate of $g$ for all $\theta$, we restrict our attention to a finite subset of the unknown parameter set $\Theta$ (discretization) that is guaranteed to contain a good action. The amount of misspecification is bounded using a union bound argument over the discretization of $\Theta$. While the discretization of $\Theta$ may eliminate the optimal arm and the function $g$ is shown to be non-smooth, we show that the function $r(\theta)=\adot{g(\theta),\thetas}$ is smooth on a neighborhood of $\thetas$. This is sufficient to show that if discretized finely, the discrete set will contain a good arm. The final worst-case regret bound can be controlled by the regret bound of the linear bandit algorithm and the batch lengths.

We next provide a high level explanation on why our reduction enables \textbf{tighter regret bounds} for contextual linear bandits with stochastic contexts. Using the Phased Elimination algorithm (PE) \cite{lattimore2020learning} with modified confidence intervals to solve the misspecified linear bandit instance, we prove a high probability bound of $O(d\sqrt{T\log T})$ on the regret of the contextual bandit problem using our reduction. Existing analysis techniques for  changing action sets  suffer $\log$ factors due to several reasons, such as bounding the radius of an ellipsoid that contains the unknown parameter with high probability, bounding the regret by the radius of the ellipsoid, or summing the regret over multiple episodes; in contrast, the $\log$ factors in our analysis appear only from a union bound used to prove concentration of the estimated arm means with high probability.

The reduction also implies improved regret bounds for \textbf{misspecified contextual linear bandits} by leveraging the optimal and nearly optimal regret bounds proved in the literature \cite{lattimore2020learning} for fixed action sets. This results in high probability bounds; as opposed to the best known results for changing action sets that add $\log$ factors and only hold in expectation. As the reduction introduces a small $\tilde{O}(1/\sqrt{T})$ misspecification, this adds to the possibly unknown misspecification $\epsilon$.
There is a little subtlety here; some algorithms adapt to known $\epsilon$ better than the unknown case. To avoid adding extra factors to the regret bounds by sub-optimal adaptation to the known $\tilde{O}(1/\sqrt{T})$ part of misspecification, we slightly modify existing algorithms to account for this. This is done by increasing the number of times an arm is explored by a constant factor.
Similarly, our results allow to carry over the better regret bounds for linear bandits with \textbf{adversarial corruption} to the contextual setting. This is achieved by modifying the algorithms for linear bandits with adversarial corruption to account for the known misspecification added by our reduction.

Our reduction has a byproduct; it limits the size of available actions in each round to the minimum between $|\calA|$ and the size of the  parameter set $\Theta$ (recall that we construct an action $g(\theta)$ for each $\theta\in \Theta$). While in  general  a discretization of size $T^{\Omega(d)}$ is required to guarantee a good action in the discrete set, if the unknown parameter follows some structure, a discretization of smaller size can be sufficient. For example, for contextual linear bandits with $s$-\textbf{sparse unknown parameter}, we show that a discretization of size $T^{O(s)}$ is sufficient. This directly implies a high probability $O(\sqrt{dsT\log T})$ regret bound as opposed to the best known $O(\sqrt{dsT}\log T)$ bound \cite{abbasi2012online} that loses extra $\log$ factors due to solving a linear regression problem over the space of sparse unknown parameters.

As it is enough to make only batch updates to our estimates of the actions $g(\theta)$, our algorithm can be modified to provide improved regret bounds for contextual linear bandits with $O(\log \log T)$ batches. We use batch lengths that were introduced in \cite{gao2019batched} which grow as $T^{1-2^{-m}}$. In each batch, as action set is fixed, we utilize elimination algorithms with the G-optimal design. However, at batch $m$, the gaps of sub-optimal actions depend on the confidence of our estimates in batch $m-1$, which rely on the G-optimal design using the estimates of $g$ from batch $m-2$. As a result, the regret in batch $m$ can be at most the ratio between the length of batch $m$ and batch $m-2$. Considering the growth rate of the batch lengths, this ratio can be large, especially in the first few batches. To fix this, we modify the batch lengths to grow in length only at the batches with odd index, while even batches use the same length as the previous batch. Our algorithm results in an $O(d\sqrt{T\log T}\log \log T)$ regret bound using $O(\log \log T)$ batches, improving a $\sqrt{\log d}$ factor (that appeared due to the distributional G-optimal design proposed in \cite{ruan2021linear}) in the regret bound over the best known result \cite{ruan2021linear}. Our result also provides the first high probability bound under the $O(\log \log T)$ batches limitation. 

As a consequence of the batch learning result and the fact that our reduction can limit the action set based on the size of $\Theta$, we provide the first algorithm with $O(\log \log T)$ batches for the sparse setting with a regret bounded by $O(\sqrt{dsT\log T}\log \log T)$ with high probability, where $s$ is the sparsity parameter.


\subsection{Paper organization} Section~\ref{sec:setup} describes our setup and reviews notation; Section~\ref{sec:red-known-unknown} describes our reduction and main theorems; and Section~\ref{sec:implications} uses our reduction to prove improved regret bounds for a number of stochastic linear bandit problems.

\section{Notation and Models}\label{sec:setup}

\textbf{Notation.} We use the following notation throughout the paper. For a vector $X$ we use $X_i$ to denote the $i$-th element of the vector $X$. The set $\{1,...,i\}$ for $i\in \bbN, i>0$ is denoted by $[i]$, where $\bbN$ is the set of natural numbers. We say that $y=O(f(x))$ if there is $x_0$ and a constant $c$ such that $y\leq cf(x)$ $\forall x>x_0$; we also use $\tilde{O}(f(x))$ to omit $\log$ factors. A $\delta$-net (with respect to norm-2) of a set $\calA\subseteq \bbR^d$ for $\delta>0$ is any set $\mathcal{B}\subseteq \bbR^d$ such that for every $a\in \calA$, there exists $b\in \mathcal{B}$ with $\|a-b\|_2\leq \delta$, where $\bbR$ is the set of real numbers. For a set of sets $\calS$, $\cup \calS$ denotes the union of all elements in $\calS$. For a family of sets $\{\calA_i\}_{i=1}^n$ we use $\prod_{i=1}^n \calA_i=\{(a_1,...,a_n)|a_i\in \calA_i, \; \forall i \in [n]\}$ to denote the product set.

\noindent\textbf{Contextual Linear Bandits.} 
We consider a contextual linear bandit problem, where a learner interacts with an environment over a time horizon of length $T$. 
At time $t\in [T]$, the learner observes a set of admissible arms $\calA_t$ representing the context, pulls an arm $a_t\in \calA_t$, and receives a reward 
\begin{equation} \label{eq_LinB}
r_t=\adot {a_t,\thetas} + \eta_t, 
\end{equation}
where the context $\calA_t$ is generated from a distribution $\calD$, $a_t$ is a function of the history $\calH_t=\{\calA_1, a_1, r_1,..., \calA_t\}$, $\thetas$ is an unknown parameter vector of dimension $d$, and $\eta_t$ is a random noise. Here, the noise $\eta_t$ follows an unknown distribution that satisfies $\bbE [\eta_t|\calF_t]=0$, $\bbE [\exp(\lambda \eta_t)|\calF_t]\leq \exp(\lambda^2/2)$  $\forall \lambda \in \bbR$ (sub-Gaussian), where $\calF_t=\sigma \{\calA_1, a_1, r_1,..., \calA_t, a_t\}$ is the filtration of all historic information up to time $t$, and $\sigma(X)$ is the $\sigma$-algebra generated by $X$. We follow the standard assumptions that $\thetas\in \Theta\subseteq \{\theta|\ \|\theta\|_2\leq 1\}$,   $\|a\|_2\leq 1$ $\forall a\in \calA_t$ and  $\forall t\in [T]$ almost surely. {The learner adopts a policy $\pi$ that maps the history up to time $t$, $(\calA_1, a_1,r_1,...,\calA_{t-1},a_{t-1},r_{t-1},\calA_t)$, to a probability distribution over $\calA_t$; we denote the policy  $\pi(\calA_1, a_1,r_1,...,\calA_{t-1},a_{t-1},r_{t-1},\calA_t)$ at time $t$ by $\pi_t(\calA_t)$.} The goal of the learner is to minimize the regret defined as
\begin{equation}\label{eq:reward-orig}
    R_T = \sum_{t=1}^T \max_{a\in \calA_t} \adot{a,\thetas} - \adot{a_t,\thetas}.
\end{equation}
In the next sections we assume for simplicity that for each $\theta\in \Theta$, there is a unique $a_t\in \calA_t$ that satisfies $\adot{a_t,\theta}=\sup_{a\in \mathcal{A}_t}\adot{a,\theta}$ almost surely. This is to avoid dealing with approximations and choice functions (if there are multiple optimal actions) in the description of our algorithms. However, our results do not need this assumption, please see Apps.~\ref{app:assump}.

\noindent\textbf{Batch Learning.} In this setting, 
the learner is allowed to change the policy $\pi_t$  only at $M$ pre-specified   time slots $1\leq t^{(1)},...,t^{(M)}\leq T$, where $M$ is the number of batches.

\noindent\textbf{Misspecified Linear Bandits.} Here pulling an action $a\in \mathcal{A}_t$  generates a reward perturbed as
\begin{equation}\label{eq:rew-misspec}
    r_t = \adot{a,\thetas}+\eta_t+f(a),
\end{equation}
where $f$ is unknown perturbation function, $\thetas$ is the unknown parameter vector, and $\eta_t$ is a  zero-mean noise that is 1-subgaussian conditioned on the history. The amount of deviation in the rewards is upper bounded by $\epsilon= \sup_{a\in \cup \text{Supp}(\calD)}|f(a)|$, and $\epsilon$ is called the amount of misspecification, where $\calD$ is the context distribution and $\text{Supp}(\calD)$ is the support set of $\calD$.

\noindent\textbf{Adversarial Corruption.} We assume an adversary that operates as follows at each time $t$:\\
 $\bullet$ The adversary observes the history of all past contexts $\calA_1,...,\calA_{t-1}$, actions $a_1,...,a_{t-1}$, rewards $r_1,...,r_{t-1}$, and previously corrupted rewards $\tilde{r}_{1},...,\tilde{r}_{t-1}$, together with the current context $\calA_t$.\\
 $\bullet$ The adversary decides on a corruption function $c_t:\calA_t\to \bbR$ that determines the amount of corruption for each action.\\
  $\bullet$  The learner observes the history of contexts $\calA_1,...,\calA_{t}$, actions $a_1,...,a_{t-1}$, and previously corrupted rewards $\tilde{r}_1,...,\tilde{r}_{t-1}$.\\
$\bullet$ The learner pulls arm $a_t$ and observes reward $\tilde{r}_t = r_t+c_t(a_t)$, where $c_t(a_t)$ is the corruption provided by the adversary.\\
Note that the true reward $r_{t}$ follows the linear bandits model, while the corrupted rewards do not need to.  We assume that the amount of corruption the adversary can inflict is bounded as
\begin{equation}
    \sum_{t=1}^T \sup_{a\in \calA_t} |c_t(a)|\leq C,
\end{equation}
where $C$ is the maximum amount of corruption.

\noindent\textbf{Sparsity.} 
We here assume that the $d$-dimensional parameter vector $\thetas$ in (\ref{eq_LinB}) is sparse, namely  $\|\thetas\|_0\leq s$ for some known $s\in [d]$, where  $\|\thetas\|_0$ denotes the norm-0 or cardinality of the vector $\thetas$.

\section{Reduction from Stochastic Contextual to Linear Bandits }\label{sec:red-known-unknown}
\subsection{Reduction for Known Context Distribution $\calD$}\label{sec:reductions}
We construct a contextual linear bandit algorithm $\calM$ that operates at a high level as follows.
At each time $t$, the learner:\\
 {\bf  Step 1 (plays):}   observes a set of actions $\calA_t$, 
    uses   $\theta_t$ (i.e., the current estimate of $\theta^*$) to decide which action $a_t\in \calA_t$ to play, and observes the associated reward $r_t$;\\
{\bf Step 2 (learns):} calculates $\theta_{t+1}$,  i.e., an updated  estimate of $\theta^*$. \\ 
In our reduction, we use a single-context algorithm $\Lambda$ for learning the parameter $\thetas$ in step 2, i.e., choose $\theta_{t+1}$, and prove that we can achieve the same worst-case regret bound as $\Lambda$.

\paragraph{\bf Fixed action space.} 
We provide to the linear bandit algorithm $\Lambda$  the  fixed action space 
\begin{equation}\label{eq:fixed-actions} 
\calX=\{g(\theta)|\theta \in \Theta\},
\mbox{ where  }    g(\theta) = \bbE_{\calA_t\sim \calD}  [\arg\max_{a\in \calA_t} \adot{a,\theta}|\theta].
\end{equation}
That is,
for each  $\theta$, we create an action $g(\theta)$ that is the expected best action (under the distribution of the context $\calD$) with respect to the parameter $\theta$. 
We illustrate using a simple example.

{\em Example 1.} Assume that we may observe one out of the two following action sets $\{[1],[-1]\},\{[1]\}$ randomly with probability $1/2$. The function $g(\theta):\bbR\to \bbR$ can be calculated as follows
\begin{equation}
    g(\theta)= \left\{
	\begin{array}{ll}
		1  & \mbox{if } \theta \geq 0 \\
		0 & \mbox{if } \theta < 0
	\end{array}
\right.
\end{equation}
and thus $\calX=\{[ 0 ],[ 1 ]\}.$

\paragraph{Reduction Algorithm.} The algorithm $\calM$   proceeds  at each time $t$ as follows:\\ 
    1. The single context algorithm $\Lambda$ decides to play an action $g(\theta_t)\in \calX$, using  the history $\Lambda$ observed. This action is never actually played. 
     Instead, $\calM$ observes what action $g(\theta_t)$ that
    $\Lambda$ selected, and uses the associated $\theta_t$ as its current  estimate of $\thetas$.\\
    2. $\calM$ observes
    $\calA_t$, plays the action
$a_t=\arg\max_{a\in \calA_t}\adot{a,\theta_t}$ and receives reward $r_t$.
It provides this reward to $\Lambda$. \\
    3. $\Lambda$  assumes that the reward $r_t$ it received was generated according to  the linear bandit model
    $r_t = \adot{g(\theta_t),\thetas}+\eta'_t,$
    and adds the action-reward pair $(g(\theta_t),r_t)$ to its history.

Note that the set of actions $\calX$ we created contains actions that may not be part of the original sets $\calA_t$'s; this is fine, since these actions are actually never played - they are used to simulate an environment that enables $\Lambda$ to correctly update its estimate of $\thetas$.  That is,  although all actions played come from the eligible sets $\calA_t$, all learning (updates on $\theta_t$) is derived from the single context algorithm $\Lambda$  that never explicitly learns $\calA_t$.

\begin{theorem}\label{thm:reduction}
     Let $\Lambda$ be any algorithm for linear bandits and $I$ be a contextual linear bandit instance with stochastic contexts, unknown parameter $\thetas$ and rewards $r_t$ generated as described in the reduction algorithm described above. It holds that\\
     $\bullet$ The reward $r_t$ is generated, by pulling the arm $g(\theta_t)$, from a linear bandit instance $L$ with action set $\calX$, and unknown parameter $\thetas$.\\
         $\bullet$ The reduction results in an algorithm $\calM$ for contextual linear bandits such that with probability at least $1-\delta$ we have
     \begin{equation} \label{eq-th1}
         |R_T^\calM(I)-R_T^\Lambda(L)|\leq c\sqrt{T\log(1/\delta)},
     \end{equation}
    where $R^\Lambda_T(L)$ is the regret of $\Lambda$ over the constructed linear bandit instance $L$, $R_T^\calM(I)$ is the regret of algorithm $\calM$ over the instance $I$ and $c$ is a universal constant.
\end{theorem}

\noindent{\em Proof Outline.} The complete proof is provided in App. A; we here provide a brief outline.  
The basic idea is to show that the action taken by algorithm $\calM$ at time $t$ is an unbiased estimate of $g(\theta_t)$, then decompose  $R_T^\calM(I)$ as
\begin{align}\label{eq:reg_init}
	R_T^\calM(I)&= \sum_{t=1}^T \adot{\arg\max_{a\in \calA_t} \adot{a,\thetas},\thetas} - \mathbb{E}[\adot{\arg\max_{a\in \calA_t} \adot{a,\thetas},\thetas}] \nonumber \\
	&\qquad +\mathbb{E}[\adot{\arg\max_{a\in \calA_t} \adot{a,\theta_t},\thetas}|\theta_t] - \adot{\arg\max_{a\in \calA_t} \adot{a,\theta_t},\thetas} 
	 +\adot{g(\thetas),\thetas} - \adot{g(\theta_t),\thetas},
\end{align}
where the expectation is with respect to the randomness in the context generation $\calA_t$. We then prove and use results such as 
$    \adot{g(\theta'),\theta'}=\max_{\theta \in \Theta} \adot{g(\theta),\theta'}, \;\forall \; \theta'\in \Theta,
$
 to arrive at
\begin{align}\label{eq:diff_bound}
    |R_T^\calM(I)-R_T^\Lambda(L)|
 &\leq |\sum_{t=1}^T \adot{\arg\max_{a\in \calA_t} \adot{a,\thetas},\thetas} - \mathbb{E}[\adot{\arg\max_{a\in \calA_t} \adot{a,\thetas},\thetas}]| \nonumber \\
	&\qquad +|\sum_{t=1}^T\mathbb{E}[\adot{\arg\max_{a\in \calA_t} \adot{a,\theta_t},\thetas}|\theta_t] - \adot{\arg\max_{a\in \calA_t} \adot{a,\theta_t},\thetas}|,
\end{align}
Next, we show that the quantity
$
    \Sigma_T := \sum_{t=1}^T \mathbb{E}[\adot{\arg\max_{a_\in \calA_t} \adot{a,\theta_t},\thetas}|\theta_t] - \adot{\arg\max_{a_\in \calA_t} \adot{a,\theta_t},\thetas}
$
is a martingale {(using that $a_t$ is unbiased estimate of $g(\theta_t)$)} with a bounded difference (by boundedness of $\Theta, \calA_t$) and apply the Azuma–Hoeffding inequality to bound these terms with high probability.  To prove that the rewards $r_t$ come from a linear bandit instance $L$, we finally show that the reward $r_t$ can be expressed as $r_t = \adot{g(\theta_t),\thetas}+\eta'_t$, where $\eta'_t$ is a zero mean $1$-sub-Gaussian noise conditioned on the filtration of historic information $\theta_1,...,\theta_t$ and rewards $r_1,...,r_{t-1}$ of the instance $L$.} \hfill{$\square$}

\begin{algorithm} 
	\caption{Reduction from stochastic contexts to no context}\label{alg:general}
	\begin{algorithmic}[1]
		\State Input: confidence parameter $\delta$, phase lengths $\{t^{(m)}\}_{m=1}^{M+1}$, and algorithm ${\Lambda_\epsilon}$ for linear contextual bandits with $\epsilon$ misspecification.
        \State Initialize: $g^{(1)}:\Theta'\to \bbR^d$ randomly, $\epsilon_1=1$, and let $\calX_1=\{g^{(1)}(\theta)|\theta \in \Theta'\}$.
		\For{$m = 1:M$}
        \For{$t=t^{(m)}+1,...,t^{(m+1)}$}
		\State Let $g^{(m)}(\theta_t)\in \calX_m$ be the arm selected by $\Lambda_{\epsilon_m}$ after observing rewards $r_{t^{(m)}+1},...,r_{t-1}$.
		\State Play arm $a_t=\arg\max_{a\in \calA_t}\adot{a,\theta_t}$ and receive reward $r_t$. Provide $r_t$ to ${\Lambda_{\epsilon_m}}$.
		\EndFor
        \State Update: $g^{(m+1)}(\theta)=\frac{1}{t^{(m+1)}}\sum_{t=1}^{t^{(m+1)}}\arg\max_{a\in \calA_t}\adot{a,\theta}$, $\calX_{m+1}=\{g^{(m+1)}(\theta)|\theta \in \Theta'\}$, and $\epsilon_{m+1} = 2\sqrt{\log (M|\Theta'|/\delta)/t^{(m+1)}}$
        \EndFor
	\end{algorithmic}
\end{algorithm}

\subsection{Reduction for Unknown Context Distribution}\label{sec:main-alg}
As described in (\ref{eq:fixed-actions}), calculating the function $g(\theta)$ requires knowledge of the distribution $\calD$. 
We do not really need this knowledge;  we prove it is sufficient to use empirical estimation of $g(\theta)$. 
 As a result, we prove that any stochastic linear contextual bandit instance, even for unknown context distributions, can be reduced to a linear bandit instance albeit with a small misspecification.

Our basic approach follows the reduction in Section~\ref{sec:reductions} but uses a sequence of functions $g^{(m)}$ that approximate  $g$ increasingly well (as $m$ increases).
To do so, as it is hard to guarantee a good estimate of $g$ for all $\theta$, we restrict our attention to a finite subset of $\Theta$ that is large enough to include a good action.
In particular, instead of considering actions in a continuous space $\Theta$, we only consider a finite subset of actions $\Theta'\subseteq \Theta$, where $\Theta'$ is an $1/T$-net for $\Theta$ according to the norm-2 distance. 
We  divide the time horizon $T$ into $M$ epochs, each of duration $T_m=t^{(m+1)}-t^{(m)}$, $m=1\ldots M$. For each epoch, we construct an empirical estimate of $g$ as $g^{(m)}(\theta)=\frac{1}{t^{(m)}}\sum_{t=1}^{t^{(m)}} (\arg\max_{a\in \calA_t}\adot{a,\theta})$, and calculate the set of actions  $\calX_{m}=\{g^{(m)}(\theta)|\theta \in \Theta'\}$. We then use at each epoch a single context algorithm as before, but we now provide at epoch $m$  the fixed set of actions $\calX_{m}$ for the algorithm to choose from. As a result, the regret of the linear bandit problem is defined as
\begin{equation}
     R^{\Lambda_\epsilon}_T = \sum_{t=1}^T \max_{\theta \in \Theta'} \adot{g(\theta),\thetas} - \adot{g(\theta_t), \thetas}.
\end{equation}

Our algorithm relies on $\Lambda_\epsilon$, an algorithm for linear bandits with $\epsilon$ misspecification. The misspecification reflects our confidence in our estimate of the function $g$, hence, decreases each epoch. We start with a large value of the misspecification parameter $\epsilon_1=1$ and a random initialization of the function $g$ (which we cal $g^{(1)}$), and hence random initialization of $\calX$ denoted $\calX_1$. At time slot $t$ of epoch $m$, the algorithm asks $\Lambda_{\epsilon_m}$ for an action to play $g^{(m)}(\theta_t)\in \calX_m$ given the history of action and rewards $\{(g^{(m)}(\theta_i),r_i)\}_{i=t^{(m)}+1}^{t-1}$ in epoch $m$ only. The algorithm pulls the action $a_t=\arg\max_{a\in \calA_t}\adot{a,\theta_t}$ and observes $r_t$. The reward $r_t$ is then passed to the algorithm $\Lambda_{\epsilon_m}$. At the end of each epoch, the misspecification parameter, estimates $g^{(m)}$, and action set $\calX_m$ are updated. The pseudo-code of our reduction is provided in Algorithm~\ref{alg:general}, and the proof in App. B. {To achieve nearly optimal regret bounds, the misspecification $\epsilon$ need to be $\tilde{O}(1/\sqrt{T})$ (recall the $\Omega(\sqrt{d}\epsilon T)$ regret lower bound \cite{lattimore2020learning}). Attempting to use $M=2$ to first estimate $g$ and then learn in the second epoch, would require the length of the first epoch to be $\Omega(T)$, to ensure the $\tilde{O}(1/\sqrt{T})$ misspecification, resulting in linear regret. Instead, as we clarify next, we use exponentially increasing epoch lengths to mix the learning with a gradual estimation of $g$ resulting in misspecification that is effectively $\tilde{O}(1/\sqrt{T})$.}
 
\begin{theorem}\label{thm:main}
     Let $\Lambda_\epsilon$ be an algorithm for linear bandits with $\epsilon$ misspecification and $I$ be a contextual linear bandit instance with stochastic contexts, unknown parameter $\thetas$ and  rewards $r_t$ are generated as described in Algorithm~\ref{alg:general}. The following holds:\\
         $\bullet$ Conditioned on $\calH_{t^{(m)}}=\sigma(\calA_1,a_1,r_1,...,\calA_{t^{(m)}},a_{t^{(m)}},r_{t^{(m)}})$: in epoch $m$, the rewards $r_t$ are generated, by pulling arm $g^{(m)}(\theta_t)$, from a misspecified linear bandit instance $L_m$ for $t=t^{(m)}+1,...,t^{(m+1)}$, action set $\calX_m=\{g^{(m)}(\theta)|\theta\in \Theta'\}$, unknown parameter $\thetas$, mean rewards $\adot{g(\theta),\thetas}$, and unknown misspecification $\epsilon'_m$.\\
         $\bullet$ The misspecification $\epsilon'_m$ is bounded by the known quantity $\epsilon_m$ with probability at least $1-c\delta/M$.\\
         $\bullet$ With probability at least $1-\delta$ we have that $|R_T(I)-\sum_{m=1}^MR^{\Lambda_{\eps_m}}_{T_m}(L_m)|\leq c\sqrt{T\log(1/\delta)}$, where $R_T(I)$ is the regret of Algorithm~\ref{alg:general} over the instance $I$, $R^{\Lambda_{\eps_m}}_{T_m}(L_m)$ is the regret of algorithm $\Lambda_{\epsilon_m}$ over the bandit instance $L_m$ in phase $m$, $T_m=t^{(m+1)}-t^{(m)}$, and $c$ is a universal constant.
\end{theorem}
{As a consequence, we prove the following corollary in Section~\ref{sec:implication-tighter}.
\begin{corollary}
    For Algorithm~\ref{alg:general} with $t^{(m)}=2^{m-1}$ and $\Lambda_\epsilon$ to be PE with modified confidence intervals \cite{lattimore2020learning}, it holds that with probability at least $1-c\delta$ we have that $R_T=O(d\sqrt{T\log(T/\delta)})$.
\end{corollary}}

\begin{algorithm} 
	\caption{Reduction from stochastic contexts to no context for product context sets}\label{alg:prod-set}
	\begin{algorithmic}[1]
		\State Input: an algorithm ${\Lambda}$ for linear contextual bandits with action set
  \begin{equation}\label{eq:action-set-prod}
      \calX = \{a'\in \{0,1\}^{2d}|\ a'_{2i-1}+a'_{2i}=1\ \forall i\in [d]\}
  \end{equation}
		\For{$t = 1:T$} 
		\State Ask $\Lambda$ for an arm to play $a'_t\in \calX$ given the history $r_1,...,r_{t-1}$.
		\State Play arm $a_t$ with $(a_t)_i=\max\{a\in \calA_t^{(i)}\}$ if $a'_{2i-1}=1$ and $(a_t)_i=\min\{a\in \calA_t^{(i)}\}$ otherwise. Receive reward $r_t$.
		\EndFor
	\end{algorithmic}
\end{algorithm}
\subsection{Reduction When  Context is a Product Set}

In this section, we consider the special case where the action distribution  $\calD$ is unknown, but the action space has a coordinate-wise product structure, i.e.,
\begin{assumption}\label{assump:prod-set}
$\calA_t=\prod_{i=1}^d \calA_t^{(i)}$, where $\calA_t^{(i)}\subset \bbR$.
\end{assumption}
\noindent This is an important hard instance that appears in many lower bounds \cite{lattimore2020bandit}. We will show that in this case, the $d$-dimensional stochastic contextual bandits can be reduced to a linear bandit problem with no misspecification, although the distribution $\calD$ is unknown, but where now the parameter vector is over $2d$ dimensions.

The main idea of the reduction is that $g(\theta)$ can be factored into a coordinate-wise product between an unknown vector that only depends on the context distribution and a known vector that only depends on $\theta$. The unknown vector can be incorporated with $\thetas$ reducing the contextual instance to a linear bandit instance but now with a different unknown parameter $\thetas'$.
In particular, we can write $\adot{g(\theta), \thetas} = \adot{a'(\theta), \thetas'}$, where $a'(\theta)$ is a vector in $\bbR^{2d}$  that does not depend on the distribution, and equals
    \begin{equation}
        (a'(\theta))_{2i} = \left\{
	\begin{array}{ll}
		1  & \mbox{if } (\theta)_i< 0 \\
		0 & \mbox{otherwise}
	\end{array}
	\right. \mbox{ and } \ (a'(\theta))_{2i-1}=1-(a'(\theta))_{2i} \quad \forall i=1,\ldots,d, .
    \end{equation}
Thus, we can 
follow the same reduction algorithm as in Section~\ref{sec:reductions}, but where now we call a $2d$-dimensional  linear bandit algorithm $\Lambda$ and provide $\Lambda$ with the  fixed action set
\begin{equation}\label{eqa}
    \calX = \{a'\in \{0,1\}^{2d}|\ a'_{2i-1}+a'_{2i}=1\ \forall i\in [d]\}
    \end{equation}
as the pseudocode Algorithm~\ref{alg:prod-set}
describes.
\begin{theorem} \label{thm:product}
    Let $\Lambda$ be any algorithm for linear bandits, $I$ be a contextual linear bandit instance with stochastic contexts that satisfy Assumption~\ref{assump:prod-set} with unknown parameter $\thetas$, and $r_t$ be the rewards generated as described in Algorithm~\ref{alg:prod-set}. It holds that\\
    $\bullet$ The rewards $r_t$ are generated, by pulling arm $a'_t$, from a linear bandit instance $L$ with action set $\calX$ in \eqref{eq:action-set-prod}, and unknown parameter $\thetas'\in \bbR^{2d}$ with $\|\thetas'\|_2\leq 2$.\\ 
    $\bullet$ With probability at least $1-\delta$ it holds that
     \begin{equation} \label{eq-th1}
         |R_T(I)-R_T^\Lambda(L)|\leq c\sqrt{T\log(1/\delta)},
     \end{equation}
    where $R_T(I)$ is the regret of Algorithm~\ref{alg:prod-set} over the instance $I$ and $R_T^\Lambda(L)$ is the regret of algorithm $\Lambda$ over the instance $L$, and $c$ is a universal constant.
\end{theorem}

\begin{proof}
    The proof follows from Theorem~\ref{thm:reduction} using the observation that
    \begin{equation}\label{eq:arms-prod}
        (g(\theta))_i = \left\{
	\begin{array}{ll}
		\bbE_{\calA \sim \calD} [\max_{a\in \calA}(a)_i]  & \mbox{if } (\theta)_i< 0 \\
		\bbE_{\calA \sim \calD} [\min_{a\in \calA}(a)_i] & \mbox{if } (\theta)_i> 0
	\end{array}
	\right..
    \end{equation}
    and thus, given (\ref{eqa}), we can  
    define $\thetas'\in \bbR^{2d}$ as
    \begin{equation}
        (\thetas')_i = \left\{
        \begin{array}{ll}
    		\bbE [\max_{a\in \calA}(a)_{(i+1)/2}](\thetas)_{(i+1)/2}  & \mbox{if $i$ is odd} \\
    		\bbE [\min_{a\in \calA}(a)_\flr{(i+1)/2}](\thetas)_\flr{(i+1)/2} & \mbox{if $i$ is even.}
    	\end{array}
	\right.
    \end{equation}
    
    By definition of $a'(\theta), \thetas'$, we have that
    $\adot{g(\theta), \thetas} = \adot{a'(\theta), \thetas'}$.
    To see that $\|\thetas'\|\leq 2$, we observe that
    \begin{equation}
        \sum_{i=1}^d|(\thetas')_{2i}| = \sum_{i=1}^d|\bbE [\max_{a\in \calA}(a)_{i}](\thetas)_{i}|\leq \adot{|g(\ind /\sqrt{d})|,|\thetas|}\leq 1,
    \end{equation}
    where absolute value of a vector is defined as a vector with $(|\thetas|)_i=|(\thetas)_i|$. Similarly, $\sum_{i=1}^d|(\thetas')_{2i-1}|\leq 1$. Hence, $\|\thetas'\|_2\leq \|\thetas'\|_1\leq 2$.
    
As before, to construct the linear bandit instance, we need rewards that follow the stochastic linear bandits model. The result follows from Proposition~\ref{prop:rewrite-reward} and the fact that $\adot{g(\theta_t), \thetas} = \adot{a_t, \thetas'}$.
\end{proof}
\section{Implications}\label{sec:implications}
\subsection{Tighter Regret Bound for Contextual Linear Bandits}\label{sec:implication-tighter}
In this subsection, we show that our reduction leads to the first $O(d\sqrt{T\log T})$ high probability upper bound for linear bandits with changing action sets. We rely on the Phased Elimination (PE) \cite{lattimore2020learning} as our linear bandit algorithm $\Lambda_\epsilon$. PE is the same as Algorithm~\ref{alg:batched}, that we will use next in the batched setting, except that $\calX, g$ are fixed (recall that we apply PE within an epoch that fixes the estimate of $g$). The parameter $\gamma_m$ in Algorithm~\ref{alg:batched} is called the confidence interval; we will specify its value in our theorems. In Algorithm~\ref{alg:general}, setting $t^{(m)}=2^{m-1}$ and $\Lambda_\epsilon$ to be PE with modified confidence intervals \cite{lattimore2020learning} to account for the misspecification we get the following corollary.
\begin{corollary}\label{cor:implication-tighter}
    For Algorithm~\ref{alg:general} with $t^{(m)}=2^{m-1}$ and $\Lambda_{\epsilon_m}$ to be PE with $\gamma_m=6\sqrt{d\log (T|\Theta'|/\delta)/t^{(m)}}$ it holds that with probability at least $1-c\delta$ we have that $R_T\leq c\sqrt{dT\log(T|\Theta'|/\delta)}$.
\end{corollary}
\begin{proof}
    Let the length of phase $m$ be $T_m=t^{(m+1)}-t^{(m)}=2^{m-1}$. Conditioned on the event that the misspecification in phase $m$ is bounded by $\epsilon_m$, PE with modified confidence intervals achieves a regret $R^{\Lambda_{\eps_m}}_{T_m}$ that is upper bounded by $O(\sqrt{dT_m\log(T_m M|\Theta'|/\delta))}+\sqrt{d}T_m\epsilon_m)$ with probability at least $1-\delta/M$. Hence, by Proposition~\ref{prop:misspec} and the union bound we have that it holds with probability at least $1-c\delta$ that
    \begin{equation}
    R^{\Lambda_{\eps_m}}_{T_m}\leq c(\sqrt{dT_m\log(T|\Theta'|/\delta)}+\sqrt{d}T_m\epsilon_m) \forall m\in [M],
    \end{equation}
    where we used the fact that $T_m\leq T, M\leq T$. Substituting the value of $\epsilon_m$ we get that the following holds with probability at least $1-c\delta$
    \begin{align}
        R^{\Lambda_{\eps_m}}_{T_m}&\leq c(\sqrt{dT_m\log(T|\Theta'|/\delta)}+\sqrt{d}T_m\sqrt{\frac{\log (M|\Theta'|/\delta)}{t^{(m)}}})\nonumber \\ 
        &=c(\sqrt{dT_m\log(T|\Theta'|/\delta)}+\sqrt{dT_m\log (M|\Theta'|/\delta)}) \forall m\in [M].
    \end{align}
    Hence, using $M\leq T$, we get that with probability at least $1-c\delta$ we have
    \begin{equation}
        R^{\Lambda_{\eps_m}}_{T_m}\leq c\sqrt{dT_m\log(T|\Theta'|/\delta)}
    \end{equation}
    Substituting in Theorem~\ref{thm:main}, we get that with probability at least $1-c\delta$ it holds that
    \begin{align}\label{eq:bnd-exp-inc}
        R_T&\leq c\sqrt{T\log T}+c\sqrt{d\log(T|\Theta'|/\delta)}\sum_{m=1}^{\log T} \sqrt{T_m}\nonumber \\
        &\leq c'\sqrt{dT\log(T|\Theta'|/\delta)}\sum_{m=1}^{\log T} \sqrt{2^{m-\log T}} \nonumber \\
        &\leq c'\sqrt{dT\log(T|\Theta'|/\delta)}\sum_{i=0}^{\infty} \sqrt{2^{-i}} \nonumber \\
        &\leq \frac{c'}{1-1/\sqrt{2}}\sqrt{dT\log(T|\Theta'|/\delta)}.
    \end{align}
        
\end{proof}
It is well known that if $\Theta \subseteq \{a\in \bbR^d|\|a\|_2\leq 1\}$, then there is $1/T$-net of $\Theta$, $\Theta'$, such that $|\Theta'|\leq (6T)^{d}, \Theta'\subseteq \Theta$. This directly implies that $R_T=O(d\sqrt{T\log(T)})$ with probability at least $1-1/T$. This improves a factor of $\sqrt{\log T}\log \log T$ over \cite{li2021tight} and a factor of $\sqrt{\log T}$ over \cite{abbasi2011improved}.

\subsection{Batch Learning}
In this subsection, we show that our reduction can be applied to improve the result of \cite{ruan2021linear} for contextual linear bandits with stochastic contexts by providing a regret upper bound of $O(d\sqrt{T\log T}\log\log T)$ with high probability as opposed to the \textit{in expectation} $O(d\sqrt{T\log d\log T}\log\log T)$ regret bound in \cite{ruan2021linear}.
This can be achieved by replacing $\Lambda_{\epsilon_m}$ with the $G$-optimal design policy constructed using $\calX_m$. To compute $t^{(i)}$ we first define
\begin{equation}
    u_m = T^{1-2^{-m}}, m=1,...,M/2.
\end{equation}
We let
\begin{equation}\label{eq:batch-lengths}
    t^{(m)} = \floor{u_{m//2+1}}\forall m\in [M], t^{(M+1)}=T.
\end{equation}
where $//$ denotes integer division. For completeness, we include the pseudo-code in Algorithm~\ref{alg:batched}. The following result follows using Theorem~\ref{thm:main}.
\begin{theorem}\label{lem:batch}
    For Algorithm~\ref{alg:general} with $M=2\log \log T$ (corresponds to $2\log \log T+1$ batches), $t^{(m)}$ given in \eqref{eq:batch-lengths} and $\Lambda_{\epsilon_m}$ replaced by the $G$-optimal design policy constructed using $\calX_m$, it holds that there is a constant $c$ such that with probability at least $1-c\delta$, we have
    \begin{equation}
        R_T\leq c \sqrt{dT\log(\frac{M|\Theta'|}{\delta})}\log \log T.
    \end{equation}
\end{theorem}
\begin{proof}
    We first notice that for $M=\log \log T$, we have that $u_{m+1}/\sqrt{u_{m}}=\sqrt{T}\forall m\in [M-1]$, $u_{M+1}/\sqrt{u_{M}}=\sqrt{2T}$. Hence, we have that
    \begin{equation}\label{eq:batch-ratio}
        u_{m+1}/\sqrt{u_{m}}\leq \sqrt{2T}\forall m\in [M].
    \end{equation}
    We also have that $u_{m+1}\geq u_m\ \forall m\geq 1$ (since $u_m\leq T$).
    The proof follows from the properties of the G-optimal design together with the properties of $g^{(m)}$ in the proof of Theorem~\ref{thm:main}. The G-optimal design ensures that for any $\theta \in \Theta'$ the following holds with probability at least $1-\delta$
    \begin{equation}\label{eq:mis-effect-on-conf}
        |\adot{g^{(m)}(\theta),\hat{\theta}_m-\thetas}|\leq 2\epsilon'_m \sqrt{d}+\sqrt{\frac{4d}{T_m}\log 1/\delta},
    \end{equation}
    where $\epsilon'_m = \sup_{\theta\in \Theta'}|\adot{g^{(m)}(\theta)-g(\theta),\thetas}|$. By the triangle inequality, we have that
    \begin{align}
        |\adot{g^{(m)}(\theta),\hat{\theta}_m}-\adot{g(\theta),\thetas}|&\leq 2\epsilon'_m \sqrt{d}+\sqrt{\frac{4d}{T_m}\log 1/\delta}+|\adot{g^{(m)}(\theta)-g(\theta),\thetas}|\nonumber \\
        &\leq \epsilon'_m (2\sqrt{d}+1)+\sqrt{\frac{4d}{T_m}\log 1/\delta}
    \end{align}
    By Proposition~\ref{prop:misspec} we have that $\epsilon'_m \leq \epsilon_m \forall m\in [M]$ with probability at least $1-\delta$. Hence, by the union bound we have that the following holds with probability at least $1-\delta$
    \begin{equation}\label{eq:conf-inter}
        |\adot{g^{(m)}(\theta),\hat{\theta}_m}-\adot{g(\theta),\thetas}|\leq \epsilon_m (2\sqrt{d}+1)+\sqrt{\frac{4d}{T_m}\log(\frac{M|\Theta'|}{\delta})} \forall \theta \in \Theta' \forall m\in [M].
    \end{equation}
    Hence, with probability at least $1-\delta$, the best arm is not eliminated and the arms that are not eliminated at the end of batch $m$, will have a gap that is at most twice the confidence interval in \eqref{eq:conf-inter}, otherwise, they satisfy the elimination criterion with $\thetas$. Hence, the sum regret $\sum_{m=1}^MR^{\Lambda_{\eps_m}}_{T_m}(L_m)$ is bounded as follows with probability at least $1-\delta$
    \begin{align}
        \sum_{m=1}^M R^{\Lambda_{\eps_m}}_{T_m}(L_m)&\leq \sqrt{T}+c'\sum_{m=1}^M\epsilon_m (2\sqrt{d}+1)T_{m+1}+c'\sqrt{\frac{4d}{T_m}\log(\frac{M|\Theta'|}{\delta})}T_{m+1}\nonumber \\
        &= \sqrt{T}+c\sqrt{d\log(\frac{M|\Theta'|}{\delta})} \sum_{m=1}^M \frac{T_{m+1}}{\sqrt{T_{m-1}}}+\frac{T_{m+1}}{\sqrt{T_{m}}}\nonumber \\
        &\stackrel{(i)}{\leq} \sqrt{T}+2c\sqrt{d\log(\frac{M|\Theta'|}{\delta})} \sum_{m=1}^M \frac{T_{m+1}}{\sqrt{T_{m-1}}}+\sqrt{T}\nonumber \\
        &\stackrel{(ii)}{\leq} \sqrt{T}+2c\sqrt{d\log(\frac{M|\Theta'|}{\delta})} \sum_{m=1}^M \sqrt{T}+\sqrt{T}\nonumber \\
        &\leq c''\sqrt{dT\log(\frac{M|\Theta'|}{\delta})}\log \log T,
    \end{align}
    where $(i)$ uses \eqref{eq:batch-ratio} and the fact that either $T_{m+1}\leq \sqrt{2T_m}$ or $T_{m+1}=T_m\leq \sqrt{2TT_m}$, and $(ii)$ follows from the fact that either ($T_m=T_{m-1}$ and $T_{m+1}\leq \sqrt{2TT_m}$) or ($T_m\leq \sqrt{2TT_{m-1}}$ and $T_{m+1}=T_m$), hence, in both cases we have $T_{m+1}\leq \sqrt{2TT_{m-1}}$. The results follow by Theorem~\ref{thm:main} and the union bound.
\end{proof} 

As in Section~\ref{sec:implication-tighter}, this implies that $R_T=O(d\sqrt{T\log(T)}\log\log T)$ with probability at least $1-1/T$, which strengthens the result in \cite{ruan2021linear} to high probability result and also improves a factor of $\log d$ in the regret bound. Moreover, as will be seen in the following subsections, if $\thetas$ is $s$-sparse, Theorem~\ref{lem:batch} implies a regret bound of $O(\sqrt{dsT\log T}\log \log T)$ with probability at least $1-1/T$. To the best of our knowledge, this is the first nearly optimal algorithm for sparse unknown parameters that uses $O(\log \log T)$ batches.

\begin{algorithm} 
	\caption{Batched algorithm for linear bandits with stochastic context}\label{alg:batched}
	\begin{algorithmic}[1]
		\State Input: confidence parameter $\delta$, and phase lengths $\{t^{(m)}\}_{m=1}^M$.
        \State Initialize: $g^{(1)}:\Theta'\to \bbR^d$ randomly, $\Theta_1=\Theta'$, and let $\calX_1=\{g^{(1)}(\theta)|\theta \in \Theta'\}$.
		\For{$m = 1:M$}
        \State Find design $\rho:\calX_m\to [0,1]$ with $\max_{a\in \text{Supp}(\rho)}\|a\|^2_{G^{-1}(\rho)}\leq 2d$, $|\text{Supp}(\rho)|\leq 4d\log\log d + 16$, where $G(\rho)=\sum_{a\in\text{Supp}(\rho)}\rho(a) aa^T$.
        \State Compute $u(x)=\ceil{\rho(x)T_m}$ and $u=\sum_{x\in \text{Supp}(\rho)}u(x)$.
        \State Use the policy described by $a_t=\arg\max_{a\in \calA_t} \adot{a,\theta}$ for $u(g^{(m)}(\theta))$ times for each $\theta\in \Theta_m$ with $g^{(m)}(\theta)\in \text{Supp}(\rho)$.
        \State $\hat{\theta}_m=(\sum_{g\in \text{Supp}(\rho)}u(g)gg^T)^{-1}\sum_{i=1}^ur_ig^{(m)}(\theta_i)$.
        \State Update: $\Theta_{m+1}=\left\{\theta \in \Theta_m|\max_{\theta'\in \Theta_m}\adot{\hat{\theta}_m,g^{(m)}(\theta')-g^{(m)}(\theta)}\leq \gamma_m\right\}$, $\gamma_m=10\sqrt{\frac{d}{T_{m-1}}\log (M|\Theta'|/\delta)}$.
        \State Update: $g^{(m+1)}(\theta)=\frac{1}{t^{(m+1)}}\sum_{t=1}^{t^{(m+1)}}\arg\max_{a\in \calA_t}\adot{a,\theta}$, and $\calX_{m+1}=\{g^{(m+1)}(\theta)|\theta \in \Theta'\}$.
        \EndFor
	\end{algorithmic}
\end{algorithm}

\subsection{Misspecified Contextual Linear  Bandits}
Our reduction framework can be used to improve regret bounds for misspecified contextual linear bandits. By noticing that for $\epsilon$ misspecified contextual linear bandits, the total amount of misspecification in epoch $m$ is bounded by $\epsilon_m+\epsilon$ with high probability, the following result follows directly from \cite{lattimore2020learning}.
\begin{theorem} \label{thm:mis}
    For contextual linear bandits with known misspecification that is bounded by $\epsilon$, Algorithm~\ref{alg:general} with $\Lambda_{\epsilon_m}$ being PE with $\gamma_m=6d\sqrt{\log (T)/T_m}+\epsilon \sqrt{d}$ achieves a regret bound
    \begin{equation}
        R_T \leq c(d\sqrt{T\log(T/\delta)}+\epsilon \sqrt{d}T)
    \end{equation}
    with probability at least $1-c\delta$.
\end{theorem}
If $\epsilon$ is unknown, \cite{lattimore2020learning} showed that PE achieves a regret upper bounded by $O(d\sqrt{T\log(T)}+\epsilon \sqrt{d}T\log T$ with high probability. We notice that the dependency on $\epsilon$ has an extra $\log T$ factor as compared to the known misspecification result. To avoid adding the $\log T$ factor in the term containing $\epsilon_m$, we exploit the knowledge of $\epsilon_m$ to modify the confidence intervals in PE. We note that following the same analysis in \cite{lattimore2020learning} by replacing $\epsilon$ with $\epsilon+\epsilon_m$ results in the same concentration of estimated means except for constant factors, which implies the following result.
\begin{theorem}\label{thm:mis-unknown}
    For contextual linear bandits with unknown misspecification that is bounded by $\epsilon$, Algorithm~\ref{alg:general} with $\Lambda_{\epsilon_m}$ being PE with $\gamma_m=6d\sqrt{\log (T)/T_m}$ achieves a regret bound
    \begin{equation}
        R_T \leq c(d\sqrt{T\log(T/\delta)}+\epsilon \sqrt{d}T\log T)
    \end{equation}
    with probability at least $1-c\delta$.
\end{theorem}
\begin{proof}
    The proof follows similar steps as in the proof of Proposition 5.1 in \cite{lattimore2020learning} with different constants due to the change in $\gamma_m$.
\end{proof}
We note that as PE is performed in each epoch, the number of batches is $\Omega(\log^2T)$. However, we constructed the algorithm this way only for simplicity. It is possible to perform PE once and update the estimates of $g$ at the end of each batch similar to Algorithm~\ref{alg:batched}.
\subsection{Bandits With Adversarial Corruption}
Our reduction directly extends the results for linear bandits with adversarial corruption to the contextual setting while maintaining the same regret bounds up to $\log$ factors. This leads to $\tilde{O}(d\sqrt{T}+d^{1.5}C)$ high probability regret bound as opposed to the best known $\tilde{O}(d^{4.5}\sqrt{T}+d^4C)$ regret bound. It was shown in \cite{bogunovic2021stochastic} that PE with modified confidence intervals achieves an $\tilde{O}(d\sqrt{T}+d^{1.5}C)$ high probability regret bound for unknown corruption $C$. A model selection based approach in \cite{foster2020adapting} that uses PE as a subroutine is shown to generalize this result for unknown $C$ without changing the regret bound except for $\log$ factors. By adapting the confidence intervals of PE to account for the known $\epsilon_m$ misspecification the $\tilde{O}(d\sqrt{T}+d^{1.5}C)$ high probability regret bound extends for a model with $O(\sqrt{d\log T/T})$ misspecification and known corruption $C$. We will modify the PE confidence interval to be
\begin{equation}\label{eq:conf-adv}
    \gamma_m=8d\sqrt{\log(T) / t^{(m)}}+\frac{2C(4d\log \log d + 18)}{T_m}\sqrt{8d}.
\end{equation}
The model selection approach in \cite{foster2020adapting} implies the following result.

\begin{theorem}\label{thm:ad}
    For contextual linear bandits with unknown corruption that is bounded by $C$, Algorithm~\ref{alg:general} with $t^{(m)}=2^{m-1}$ and $\Lambda_{\epsilon_m}$ being G-COBE in \cite{foster2020adapting} with PE, that uses $\gamma_m$ in \eqref{eq:conf-adv}, as a subroutine, achieves
    \begin{equation}
        R_T=\tilde{O}(d\sqrt{T}+d^{1.5}C)
    \end{equation}
    with probability at least $1-c/T$.
    \end{theorem}
\begin{proof}
    The proof follows by proving a regret bound for PE with the chosen $\gamma_m$ for known $C$. Then use the model selection result in \cite{foster2020adapting} to prove a regret bound for unknown $C$ and finally apply our main theorem to bound the final regret.

    We notice that
    \begin{align}
        |\adot{g^{(m)}(\theta),\hat{\theta}_m}-\adot{g(\theta),\thetas}|\leq |\adot{g^{(m)}(\theta),\hat{\theta}_m-\thetas}|+|\adot{g^{(m)}(\theta)-g(\theta),\thetas}|.
    \end{align}
    Using Lemma 1 in \cite{bogunovic2021stochastic} and equation \eqref{eq:mis-effect-on-conf} to bound the first term and Proposition~\ref{prop:misspec} to bound the second term, we obtain that $|\adot{g^{(m)}(\theta),\hat{\theta}_m}-\adot{g(\theta),\thetas}|\leq \gamma_m/2$ with probability at least $1-c/T$. A standard calculation in Theorem 1 in \cite{bogunovic2021stochastic} implies a regret bound of $\tilde{O}(d\sqrt{T}+d^{1.5}C)$ with high probability. Hence, Theorem 4 in \cite{wei2022model} implies a regret bound of $\tilde{O}(d\sqrt{T}+d^{1.5}C)$ with high probability for G-COBE with unknown $C$. Applying Theorem~\ref{thm:main} and proceeding as in equation \eqref{eq:bnd-exp-inc} concludes the proof.
\end{proof}

\subsection{Sparsity}
For linear contextual bandits with $s$-sparse unknown parameter, our reduction can be used to prove $O(\sqrt{dsT\log (T)})$ regret bound with high probability as opposed to the best known $O(\sqrt{dsT}\log^2T)$ regret bound. It is not hard to show that there is $1/T$ cover with size at most $(6T)^{2s+1}$ in that case; also proved below for completeness. The following result directly follows from Corollary~\ref{cor:implication-tighter}.
\begin{theorem} \label{thm:sp}
    For Algorithm~\ref{alg:general} with $t^{(m)}=2^{m-1}$ and $\Lambda_{\epsilon_m}$ to be PE with $\gamma_m=6\sqrt{2ds\log (T/\delta)/t^{(m)}}$ it holds that with probability at least $1-c\delta$ we have that $R_T=O(\sqrt{dsT\log (T/\delta)})$.
\end{theorem}
\begin{proof}
    By Corollary~\ref{cor:implication-tighter}, we only need to show that $\Theta$ contains an $1/T$-net with size at most $(6T)^{2s+1}$. We have that there is $1/2T$-net for $\{\phi\in \bbR^s|\|\phi\|_2\leq 1\}$ with size at most $(6T)^s$. To construct $s$-sparse vectors in $\bbR^d$, there is $\sum_{i=1}^s {d \choose i}$ ways to choose the non-zero entries. This implies that the set $\{\theta\in \bbR^d| \|\theta\|_2\leq 1, \|\theta\|_0\leq s\}$ has $1/2T$-net $\mathcal{N}$ of size at most 
    \begin{equation}
        |\mathcal{N}|\leq(6T)^s\sum_{i=1}^s {d \choose i}\leq (6T)^s\sum_{i=1}^s d^i\leq (6T)^s d^{s+1}\leq (6T)^{2s+1}.
    \end{equation}
    We next construct an $1/T$-net that is subset of $\Theta$ (recall that this is required for Theorem~\ref{thm:main}). For every $x \in \mathcal{N}$, let $\mathcal{N}_x=\{\theta\in \Theta| \|\theta-x\|\leq 1/2T\}$. Let $\alpha:\mathcal{N}\to |\mathcal{N}|$ be an ordering of the set $\mathcal{N}$. And let $\mathcal{N}'\subseteq \mathcal{N}$ be defined as $\mathcal{N}'=\{x\in \mathcal{N}|\mathcal{N}_x\not\subseteq \cup_{y:\alpha(y)\leq \alpha(x)}\mathcal{N}_y\}$. Hence, $\{\mathcal{N}_x|x\in \mathcal{N}'\}$ is a set of pairwise disjoint sets. By the axiom of choice, there is a set $\mathcal{N}''$ such that $\mathcal{N}''$ contains exactly one element of each set $\mathcal{N}_x\forall x\in \mathcal{N}'$. By construction of $\mathcal{N}''$ we have that $\mathcal{N}''\subseteq \Theta$. 
    
    We also have that for each $\theta\in \Theta$, there is $x\in \mathcal{N}$ with $\|x-\theta\|_2\leq 1/2T$. Then for each $y\in \mathcal{N}_x$, we have by the triangle inequality that $\|y-\theta\|_2\leq 1/T$. By the definition of $\mathcal{N}'$, we have that $\cup \mathcal{N}'=\cup_{z\in \mathcal{N}}\mathcal{N}_z$. Then, by construction of $\mathcal{N}''$, there is $z\in \mathcal{N}''$ such that $z\in \mathcal{N}_x$. Hence, $\|z-\theta\|_2\leq 1/T$. This implies that $\mathcal{N}''\subseteq \Theta$ is an $1/T$-net of $\Theta$.
    We also have by construction that also that $|\mathcal{N}''|\leq |\mathcal{N}|\leq (6T)^{2s+1}$.
\end{proof}
\subsection{Structured Unknown Parameters}
In some cases, the dimension $d$ can be large but the unknown parameter is mapped from a small space of dimension $s$. In particular let $f:\bbR^s\to \bbR^d$ be such that $\Theta\subseteq \{f(\phi)|\|\phi\|_2\leq 1\}$ and
\begin{equation}
    \|f(\phi)-f(\phi')\|_2\leq c\|\phi-\phi'\|_2,
\end{equation}
where $c$ is a universal constant. Finding an $1/T$-net for $\Phi=\{\phi|f(\phi)\in \Theta\}$ implies a $c/T$-net for $\Theta$ with the same size. Since $\Phi\subseteq \{\phi\in \bbR^s| \|\phi\|_2\leq 1\}$, we have that there is an $1/T$-net for $\Phi$ that is contained in $\Phi$ with size at most $(6T)^s$. The following result immediately follows from Corollary~\ref{cor:implication-tighter}.
\begin{corollary}
    Under the considered structured unknown parameters assumption, Algorithm~\ref{alg:general} with $t^{(m)}=2^{m-1}$ and $\Lambda_{\epsilon_m}$ to be PE with $\gamma_m=6\sqrt{ds\log (T/\delta)/t^{(m)}}$ satisfies that with probability at least $1-c\delta$ we have that $R_T=O(\sqrt{dsT\log(T/\delta)})$.
\end{corollary}
\section{Conclusions}
We presented a novel reduction from stochastic contextual linear bandit problems to (fixed-context) linear bandit problems; our reduction explains why results, that are not achievable for adversarial contexts, are possible for stochastic contextual bandits, and offers a framework that can be leveraged to gain improved understanding and bounds for  contextual linear bandit problems. We illustrate the power of our approach by applying it to achieve improved bounds over a number of cases; this is not an exhaustive list, and we hope that our approach will be a useful tool to researchers in this field.
\bibliographystyle{IEEEtran}
\bibliography{Refs}
\newpage
\appendix
\begin{center}
{\bf \Large Appendices}
\end{center}
\section{Proof of Theorem~\ref{thm:reduction}}\label{app:thm1-reduction-proof}

\begin{theorem*}[Restatement of Theorem~\ref{thm:reduction}]
     Let $\Lambda$ be any algorithm for linear bandits and $I$ be a contextual linear bandit instance with stochastic contexts, unknown parameter $\thetas$ and rewards $r_t$ generated as described in the reduction in Section~\ref{sec:reductions}. It holds that\\
     $\bullet$ The reward $r_t$ is generated, by pulling the arm $g(\theta_t)$, from a linear bandit instance $L$ with action set $\calX$, and unknown parameter $\thetas$.\\
         $\bullet$ The reduction results in an algorithm $\calM$ for contextual linear bandits such that with probability at least $1-\delta$ we have
     \begin{equation} \label{eq-th1}
         |R_T^\calM(I)-R_T^\Lambda(L)|\leq c\sqrt{T\log(1/\delta)},
     \end{equation}
    where $R^\Lambda_T(L)$ is the regret of $\Lambda$ over the constructed linear bandit instance $L$, $R_T^\calM(I)$ is the regret of algorithm $\calM$ over the instance $I$ and $c$ is a universal constant.
\end{theorem*}
\begin{proof}
Following the reduction described in the section, we start by showing that $R_T^\Lambda(L)$, the regret of the algorithm $\Lambda$ on a linear bandit instance, is at most $\tilde{O}(\sqrt{T})$ away from $R_T^\calM(I)$ with high probability. 
Recall that the regret $R_T^\Lambda(L)$ is defined as
\begin{equation}\label{eq:reg-def-known}
    R_T^\Lambda(L) = \sum_{t=1}^T \max_{\theta \in \Theta} \adot{g(\theta),\thetas} - \adot{g(\theta_t), \thetas}.
\end{equation}
We notice that the function $g$ depends on the context distribution $\calD$. In the following we assume for simplicity that for each $\theta\in \Theta$, there is a unique $a_t\in \calA_t$ that satisfies $\adot{a_t,\theta}=\sup_{a\in \mathcal{A}_t}\adot{a,\theta}$ almost surely. We discuss how to remove this assumption at the end of the proof.

The regret $R_T^\calM(I)$ of the contextual algorithm can be decomposed as
\begin{align}\label{eq:reg_init}
	R_T^\calM(I)&= \sum_{t=1}^T \adot{\arg\max_{a\in \calA_t} \adot{a,\thetas},\thetas} - \adot{\arg\max_{a\in \calA_t} \adot{a,\theta_t},\thetas} \nonumber \\
    &= \sum_{t=1}^T \adot{\arg\max_{a\in \calA_t} \adot{a,\thetas},\thetas} - \adot{ \mathbb{E}[\arg\max_{a\in \calA_t} \adot{a,\thetas}],\thetas} \nonumber \\
	&\qquad +\adot{ \mathbb{E}[\arg\max_{a\in \calA_t} \adot{a,\theta_t}|\theta_t],\thetas} - \adot{\arg\max_{a\in \calA_t} \adot{a,\theta_t},\thetas} \nonumber \\
	&\qquad +\adot{ \mathbb{E}[\arg\max_{a\in \calA_t} \adot{a,\thetas}],\thetas} - \adot{ \mathbb{E}[\arg\max_{a\in \calA_t} \adot{a,\theta_t}|\theta_t],\thetas}\nonumber \\
 &= \sum_{t=1}^T \adot{\arg\max_{a\in \calA_t} \adot{a,\thetas},\thetas} - \mathbb{E}[\adot{\arg\max_{a\in \calA_t} \adot{a,\thetas},\thetas}] \nonumber \\
	&\qquad +\mathbb{E}[\adot{\arg\max_{a\in \calA_t} \adot{a,\theta_t},\thetas}|\theta_t] - \adot{\arg\max_{a\in \calA_t} \adot{a,\theta_t},\thetas} \nonumber \\
	&\qquad +\adot{g(\thetas),\thetas} - \adot{g(\theta_t),\thetas},
\end{align}
where the expectation is with respect to the randomness in the context generation $\calA_t$. In the following we will first show that $\adot{g(\thetas),\thetas}=\max_{\theta \in \Theta'} \adot{g(\theta),\thetas}$.

Indeed,we observe that $\forall \theta', \theta''\in \Theta$ we have
\begin{align}
	\max_{\theta \in \Theta}\adot{g(\theta),\theta'}\geq \adot{g(\theta'),\theta'} &= \bbE[\max_{a\in \calA_t} \adot{a,\theta'}]\nonumber \\
	&\geq \mathbb{E}[\adot{\arg\max_{a\in \calA_t} \adot{a,\theta''},\theta'}] = \adot{g(\theta''),\theta'}.
\end{align}
The above inequalities have to be met with equality since we can select
$\theta''=\arg\max_{\theta \in \Theta}\adot{g(\theta),\theta'}$ making the first and last terms equal.
Hence, we have proved that
\begin{equation}\label{eq:max_prop_g}
    \adot{g(\theta'),\theta'}=\max_{\theta \in \Theta} \adot{g(\theta),\theta'}, \quad \forall \; \theta'\in \Theta.
\end{equation}

Substituting in the last line of \eqref{eq:reg_init} using the triangle inequality, we get that

\begin{align}\label{eq:diff_bound}
    |R_T^\calM(I)-R_T^\Lambda(L)|&\stackrel{(i)}{=} |\sum_{t=1}^T \adot{\arg\max_{a\in \calA_t} \adot{a,\thetas},\thetas} - \mathbb{E}[\adot{\arg\max_{a\in \calA_t} \adot{a,\thetas},\thetas}] \nonumber \\
	&\qquad +\mathbb{E}[\adot{\arg\max_{a\in \calA_t} \adot{a,\theta_t},\thetas}|\theta_t] - \adot{\arg\max_{a\in \calA_t} \adot{a,\theta_t},\thetas}| \nonumber \\
 &\leq |\sum_{t=1}^T \adot{\arg\max_{a\in \calA_t} \adot{a,\thetas},\thetas} - \mathbb{E}[\adot{\arg\max_{a\in \calA_t} \adot{a,\thetas},\thetas}]| \nonumber \\
	&\qquad +|\sum_{t=1}^T\mathbb{E}[\adot{\arg\max_{a\in \calA_t} \adot{a,\theta_t},\thetas}|\theta_t] - \adot{\arg\max_{a\in \calA_t} \adot{a,\theta_t},\thetas}|,
\end{align}
where $(i)$ follows by definition of $R_T^\Lambda(L)$ and \eqref{eq:max_prop_g}.

We next bound the quantity
\begin{equation}
    \Sigma_T := \sum_{t=1}^T \mathbb{E}[\adot{\arg\max_{a\in \calA_t} \adot{a,\theta_t},\thetas}|\theta_t] - \adot{\arg\max_{a\in \calA_t} \adot{a,\theta_t},\thetas}.
\end{equation}
Let $\calF'_t=\sigma\{\theta_1,r_1,...,\theta_t\}$ be the filtration of all historic information of the linear bandit problem up to time $t$. we notice that
\begin{align}\label{eq:martingale-prop}
	\mathbb{E}[\Sigma_t|\calF'_t]&=\mathbb{E}[\Sigma_{t-1}|\calF'_t]+\bbE [\mathbb{E}[\adot{\arg\max_{a\in \calA_t} \adot{a,\theta_t},\thetas}|\theta_t]|\calH'_t] - \bbE[\adot{\arg\max_{a\in \calA_t} \adot{a,\theta_t},\thetas}|\calF'_t]\nonumber \\
	&=\Sigma_{t-1}.
\end{align}
Hence, $\Sigma_{t}$ is a martingale with a bounded difference (by boundedness of $\Theta, \calA_t$). By Azuma–Hoeffding inequality \cite{wainwright2019high}, we have that $|\Sigma_T|\leq c\sqrt{T\log 1/\delta}$ with probability at least $1-\frac{1}{2\delta}$. For completeness, we state a special case of the Azuma-Hoeffding inequality at the end of our proof. The summation $\Sigma'_T=\sum_{t=1}^T \adot{\arg\max_{a\in \calA_t} \adot{a,\thetas},\thetas} - \mathbb{E}[\adot{\arg\max_{a\in \calA_t} \adot{a,\thetas},\thetas}]$ can be bounded similarly. Hence, by \eqref{eq:diff_bound}, we have that with probability at least $1-\delta$
\begin{equation}\label{eq:regret-diff}
    |R_T^\calM(I)-R_T^\Lambda(L)|\leq c \sqrt{T\log 1/\delta}.
\end{equation}

We have shown that the regret of Algorithm~\ref{alg:general} over the instance $I$ is $O(\sqrt{T\log 1/\delta})$ away from $R_T^\Lambda(L)$ with probability at least $1-\delta$. It remains to show that the rewards $r_t$ generated by the described interaction with the instance $I$, are generated from $L$ by the interaction of algorithm $\Lambda$.
\begin{proposition}\label{prop:rewrite-reward}
    The reward $r_t$ can be rewritten as
    \begin{equation}
        r_t = \adot{g(\theta_t),\thetas}+\eta'_t,
    \end{equation}
    where $\bbE [\eta'_t|\calF'_t]=0, \bbE [\exp(\lambda \eta'_t)|\calF'_t]\leq \exp(2\lambda^2) \quad \forall \lambda \in \bbR$, and where $\calF'_t=\sigma \{\theta_1, r_1,..., \theta_t\}$ is the filtration of historic information up to time $t$.
\end{proposition}
\begin{proof}
We have that
\begin{align}
    r_t&=\adot {a_t,\thetas} + \eta_t = \adot{\arg\max_{b\in \calA_t} \adot{b,\theta_t}, \thetas} +\eta_t\nonumber \\
    &= \adot{g(\theta_t),\thetas} + \left(\eta_t+ \adot{\arg\max_{b\in \calA_t} \adot{b,\theta_t}, \thetas} - \adot{g(\theta_t),\thetas}\right).
\end{align}
We let $\eta'_t=\eta_t+ \adot{\arg\max_{b\in \calA_t} \adot{b,\theta_t}, \thetas} - \adot{g(\theta_t),\thetas}$. The proof that $\bbE [\eta'_t|\calF_t]=0$ follows similarly to \eqref{eq:martingale-prop}
\begin{align}
    \bbE[\eta_t'|\calF_t] &= \bbE[\eta_t|\calF_t]+\bbE [\adot{\arg\max_{b\in \calA_t} \adot{b,\theta_t}, \thetas} - \adot{g(\theta_t),\thetas}|\calF_t]\nonumber \\
    &=\bbE [\adot{\arg\max_{b\in \calA_t} \adot{b,\theta_t}, \thetas} - \adot{g(\theta_t),\thetas}|\theta_t] = 0.
\end{align}
Lastly, $\bbE [\exp(\lambda \eta'_t)|\calF_t]\leq \exp(2\lambda^2) \forall \lambda \in \bbR$ follows by boundedness of $\eta'_t$ which follows by boundedness of $\eta_t, \Theta, \calA_t$.
\end{proof}
This concludes the proof.
\begin{lemma}{[Azuma's Inequality \cite{wainwright2019high}]}
    Let $\Sigma_0, \Sigma_1,...$ be a martingale with respect to filtration $\calF_0,\calF_1,...$ such that $|\Sigma_i-\Sigma_{i-1}|\leq c$ almost surely. Then for all $\epsilon>0$, we have that
    \begin{equation}
        \mathbb{P}[|\Sigma_n-\Sigma_0|>\epsilon]\leq 2\exp(-\frac{\epsilon^2}{2nc^2}).
    \end{equation}
\end{lemma}
\end{proof}

\subsection{When $\sup_{a\in \mathcal{A}_t}\adot{a,\theta}$ is not Unique}\label{app:assump}
One solution is to choose $a_t$ as any action $\adot{a_t,\theta}\geq \sup_{a\in \mathcal{A}_t}\adot{a,\theta}-\delta$ for some $\delta>0$. The error arising from performing this step can be controlled by $\delta$, e.g., by setting $\delta = \exp(-T)$. Our proofs will follow by choosing $a_t$ as described above, using any deterministic or random choice function, as long as the action $a_t$ is measurable.

\section{Proof of Theorem~\ref{thm:main}}
\begin{theorem*}[Restatement of Theorem~\ref{thm:main}]
     Let $\Lambda_\epsilon$ be an algorithm for linear bandits with $\epsilon$ misspecification and $I$ be a contextual linear bandit instance with stochastic contexts, unknown parameter $\thetas$ and  rewards $r_t$ are generated as described in Algorithm~\ref{alg:general}. The following holds:\\
         $\bullet$ Conditioned on $\calH_{t^{(m)}}=\sigma(\calA_1,a_1,r_1,...,\calA_{t^{(m)}},a_{t^{(m)}},r_{t^{(m)}})$: in epoch $m$, the rewards $r_t$ are generated, by pulling arm $g^{(m)}(\theta_t)$, from a misspecified linear bandit instance $L_m$ for $t=t^{(m)}+1,...,t^{(m+1)}$, action set $\calX_m=\{g^{(m)}(\theta)|\theta\in \Theta'\}$, unknown parameter $\thetas$, mean rewards $\adot{g(\theta),\thetas}$, and unknown misspecification $\epsilon'_m$.\\
         $\bullet$ The misspecification $\epsilon'_m$ is bounded by the known quantity $\epsilon_m$ with probability at least $1-c\delta/M$.\\
         $\bullet$ With probability at least $1-\delta$ we have that $|R_T(I)-\sum_{m=1}^MR^{\Lambda_{\eps_m}}_{T_m}(L_m)|\leq c\sqrt{T\log(1/\delta)}$, where $R_T(I)$ is the regret of Algorithm~\ref{alg:general} over the instance $I$, $R^{\Lambda_{\eps_m}}_{T_m}(L_m)$ is the regret of algorithm $\Lambda_{\epsilon_m}$ over the bandit instance $L_m$ in phase $m$, $T_m=t^{(m+1)}-t^{(m)}$, and $c$ is a universal constant.
\end{theorem*}
\begin{proof}
Let $L_m$ be a bandit instance with actions $\calX_m=\adot{g^{(m)}(\theta)|\theta\in \Theta'}$ indexed by the set $\Theta'$, mean rewards $\adot{g(\theta),\thetas}\forall \theta \in \Theta'$, and $t\in \{t^{(m)}+1,...,t^{(m+1)}\}$. Let
\begin{equation*}
    \calH_{t^{(m)}}=\sigma\{\calA_1,a_1,r_1,...,\calA_{t^{(m)}},a_{t^{(m)}},r_{t^{(m)}}\}
\end{equation*}
be the filtration of all historic information before epoch $m$. Note that $g^{(m)}$ is defined in line $7$ of Algorithm~\ref{alg:general} and is the empirical estimate of $g$ using history $\calH_{t^{(m)}}$, hence, $g^{(m)}$ is $\calH_{t^{(m)}}$-predictable. Conditioned on $\calH_{t^{(m)}}$, we have that $L_m$ is  a misspecified linear bandit instance with misspecification, $\epsilon_m'=\sup_{\theta \in \Theta'}\adot{g(\theta)-g^{(m)}(\theta),\thetas}$. And the regret of the algorithm $\Lambda_{\epsilon_m}$ over $L_m$ is the random quantity
\begin{equation*}
    R^{\Lambda_{\eps_m}}_{T_m}(L_m):=\sum_{t=t^{(m)}+1}^{t^{(m+1)}} \adot{g(\theta),\thetas} - \adot{g(\theta_t), \thetas}.
\end{equation*}
As we have shown in Proposition~\ref{prop:rewrite-reward}, conditioned on $\calH_{t^{(m)}}$, $r_t$ is generated from $L_m$ by pulling arm $\theta_t$.
Define the event $\mathcal{E}_m=\{\epsilon'_m> \epsilon_m\}$ be the bad event that the random quantity $\epsilon'_m$ is greater than $\epsilon_m$ defined in Algorithm~\ref{alg:general}. We will show in Proposition~\ref{prop:misspec} that $\sum_{m=1}^M\mathbb{P}[\mathcal{E}_m]\leq \delta$. 

We next bound the regret of Algorithm~\ref{alg:general} in terms of the random quantities $\{R^{\Lambda_{\eps_m}}_{T_m}(L_m)\}_{m=1}^M$. By choosing $\delta$ sufficiently small, it it will be enough to bound $R^{\Lambda_{\eps_m}}_{T_m}(L_m)$ conditioned on $\calH_{t^{(m)}}$ and the good event $\mathcal{G}_m=\{\epsilon'_m\leq \epsilon_m\}$. Let us define the random quantity $R^{\Lambda_\epsilon}_T(L_\epsilon)=\sum_{m=1}^MR^{\Lambda_{\eps_m}}_{T_m}(L_m)$.
We show that $R^{\Lambda_\epsilon}_T(L_\epsilon)$ is at most $\tilde{O}(\sqrt{T})$ away from $R_T^\calM$ with high probability. By definition of $R^{\Lambda_{\eps_m}}_{T_m}(L_m)$ we have that
\begin{equation}
    R^{\Lambda_\epsilon}_T(L_\epsilon) = \sum_{t=1}^T \max_{\theta \in \Theta'} \adot{g(\theta),\thetas} - \adot{g(\theta_t), \thetas}.
\end{equation}
As in the proof of Theorem~\ref{thm:reduction}, we assume for simplicity that for each $\theta\in \Theta$, there is a unique $a_t\in \calA_t$ that satisfies $\adot{a_t,\theta}=\sup_{a\in \mathcal{A}_t}\adot{a,\theta}$ almost surely. This can be relaxed in the same way as in Theorem~\ref{thm:reduction}.

Recall that $L$ is the linear bandit instance in Theorem~\ref{thm:reduction} with access to the function $g$, and actions in $\calX=\{g(\theta)|\theta\in \Theta\}$. The regret $R^\calM_T(I)$ of the contextual algorithm can be bounded as
\begin{align}\label{eq:reg_init_gen}
	R_T^\calM(I)&\leq |R_T^\calM(I)-R_T^\Lambda(L)|+|R_T^\Lambda(L)-R^{\Lambda_\epsilon}_T(L_\epsilon)|,
\end{align}
where $R_T^\Lambda(L)$ is defined as follows
\begin{equation*}
     R_T^\Lambda(L) = \sum_{t=1}^T \max_{\theta \in \Theta} \adot{g(\theta),\thetas} - \adot{g(\theta_t), \thetas},
\end{equation*}
and $\{\theta_t\}$ are the  actions played by $\{\Lambda_{\epsilon_m}\}$.
The first term in \eqref{eq:reg_init_gen} is bounded in Theorem~\ref{thm:reduction}. In the following, we focus on bounding $|R_T^\Lambda(L)-R^{\Lambda_\epsilon}_T(L_\epsilon)|$. To that end, we first show that
\begin{equation}\label{eq:quant-effect-gen}
    |\adot{g(\thetas),\thetas}-\max_{\theta \in \Theta'} \adot{g(\theta),\thetas}|\leq 2/T.
\end{equation}

We recall from \eqref{eq:max_prop_g} that
\begin{equation}
    \adot{g(\theta'),\theta'}=\max_{\theta \in \Theta} \adot{g(\theta),\theta'}, \quad \forall \; \theta'\in \Theta.
\end{equation}
From $1/T$-net properties, we also have that there exists $\phi\in \Theta'$ such that $\|\thetas-\phi\|_2\leq 1/T$. Hence,
\begin{align}
    \max_{\theta \in \Theta'} \adot{g(\theta),\thetas} &\stackrel{(i)}{\leq} \max_{\theta \in \Theta} \adot{g(\theta),\thetas}\stackrel{(ii)}{=}\adot{g(\thetas),\thetas}\nonumber \\
    &\stackrel{(iii)}{\leq} \adot{g(\thetas),\phi}+1/T\nonumber \\
    &\leq \max_{\theta \in \Theta} \adot{g(\theta),\phi}+1/T\nonumber \\
    &\stackrel{(iv)}{=} \adot{g(\phi),\phi}+1/T \stackrel{(v)}{\leq} \adot{g(\phi),\thetas}+2/T \nonumber \\
    &\stackrel{(vi)}{\leq} \max_{\theta \in \Theta'} \adot{g(\theta),\thetas}+2/T,
\end{align}
where $(i)$ follows from $\Theta'\subseteq \Theta$, $(ii)$ follows from \eqref{eq:max_prop_g}, $(iii)$ follows from $\adot{g(\thetas),\thetas-\phi}\leq \|g(\thetas)\|_2\|\thetas-\phi\|_2\leq 1/T$, $(iv)$ follows from \eqref{eq:max_prop_g}, $(v)$ follows as in $(iii)$, and $(vi)$ follows from $\phi \in \Theta'$. Eq.~\eqref{eq:quant-effect-gen} follows. Note that in this part it is important to have $\Theta'\subseteq \Theta$.

As a result $|R_T^\Lambda(L)-R^{\Lambda_\epsilon}_T(L_\epsilon)|$ can be bounded as
\begin{align}\label{eq:sec_diff-gen}
    |R_T^\Lambda(L)-R^{\Lambda_\epsilon}_T(L_\epsilon)|&= |\sum_{t=1}^T \max_{\theta\in \Theta'}\adot{g(\theta),\thetas} - \max_{\theta\in \Theta}\adot{g(\theta),\thetas}|\nonumber\\
    &\leq \sum_{t=1}^T |\max_{\theta\in \Theta'}\adot{g(\theta),\thetas} - \max_{\theta\in \Theta}\adot{g(\theta),\thetas}|\stackrel{(i)}{\leq} 2,
\end{align}
where $(i)$ follows uses~\eqref{eq:quant-effect-gen} and \eqref{eq:max_prop_g}.

Hence, by \eqref{eq:reg_init_gen}, Theorem~\ref{thm:reduction}, \eqref{eq:sec_diff-gen} and union bound, we have that with probability at least $1-\delta$
\begin{equation}\label{eq:regret-diff}
    |R^\calM_T(I)-R^{\Lambda_\epsilon}_T(L_\epsilon)|\leq c \sqrt{T\log 1/\delta}.
\end{equation}
Since $R^{\Lambda_\epsilon}_T(L_\epsilon)=\sum_{m=1}^MR^{\Lambda_{\eps_m}}_{T_m}(L_m)$, we have proved the second part of the theorem. It remains to show that in each epoch $m$ the rewards $r_t$ are generated from the for linear bandit instance $L_m$ and to bound the amount of misspecification. By Proposition~\ref{prop:rewrite-reward}, it suffices to show the following.
\begin{proposition}\label{prop:misspec}
    For each $m\in [M]$, we have that with probability at least $1-\delta/M$ it holds that
    \begin{equation}
        |\adot{g(\theta),\theta'}-\adot{g^{(m)}(\theta),\theta'}|\leq 2\sqrt{\frac{\log (2M|\Theta'|/\delta)}{t^{(m)}}} \forall \theta\in \Theta', \theta' \in \Theta.
    \end{equation}
\end{proposition}
\begin{proof}

    Since for a fixed $\theta'$ we have that $\adot{\arg\max_{a\in \calA_t}\adot{a,\theta},\theta'}$ is $1/4$-subgaussian with mean $\adot{g(\theta),\theta'}$, we have that with probability at least $1-\frac{\delta}{M|\Theta'|^2}$ it holds that
    \begin{align}
        |\adot{g(\theta),\theta'}-\adot{g^{(m)}(\theta),\theta'}|&=|\adot{g(\theta),\theta'}-\frac{1}{t^{(m+1)}}\sum_{t=1}^{t^{(m+1)}}\adot{\arg\max_{a\in \calA_t}\adot{a,\theta},\theta'}|\nonumber\\
        &\leq 2\sqrt{\frac{\log (2M|\Theta'|/\delta)}{t^{(m)}}}.
    \end{align}
    By the union bound, the following holds with probability at least $1-\delta/M$
    \begin{equation}
        |\adot{g(\theta),\theta'}-\adot{g^{(m)}(\theta),\theta'}|\leq 2\sqrt{\frac{\log (2M|\Theta'|/\delta)}{t^{(m)}}} \forall \theta, \theta' \in \Theta'.
    \end{equation}
    Let us pick arbitrary $\theta\in \Theta', \theta' \in \Theta$. We have that there is $\phi'\in \Theta'$ such that $\|\theta'-\phi'\|_2\leq 1/T$. 
    Hence, by Cauchy-Schwartz and the triangle inequality, we have that the following holds with probability at least $1-\delta/M$
    \begin{equation}
        |\adot{g(\theta),\theta'}-\adot{g^{(m)}(\theta),\theta'}|\leq |\adot{g(\theta),\phi'}-\adot{g^{(m)}(\theta),\phi'}|+\frac{2}{T}\leq 2\sqrt{\frac{\log (2M|\Theta'|/\delta)}{t^{(m)}}}+\frac{2}{T} \forall \theta\in \Theta', \theta' \in \Theta.
    \end{equation}
\end{proof}
This concludes the proof.
\end{proof}

\end{document}